\documentclass[twoside]{article}

 \usepackage[accepted]{aistats2020}
\usepackage[round]{natbib}

\usepackage[utf8]{inputenc} 
\usepackage[T1]{fontenc}    
\usepackage[draft]{hyperref}
\usepackage{url}            
\usepackage{booktabs}       
\usepackage{amsfonts}       
\usepackage{nicefrac}       
\usepackage{microtype}      

\usepackage{amssymb,amsmath,amsfonts,bm,epsfig,graphicx,theorem,epstopdf}
\usepackage{rotating,setspace,latexsym,epsf,color,dsfont,caption,mathtools}
\usepackage{subfigure}

\usepackage{algorithm}
\usepackage{algorithmic}

\usepackage{enumitem}

\newcommand*\diff{\mathop{}\!\mathrm{d}}

\newcommand{\diag}{\text{diag}}

\newcommand{\ev}{\mathbf{e}}

\newcommand{\xv}{\mathbf{x}}

\newcommand{\uv}{\mathbf{u}}
\newcommand{\yv}{\mathbf{y}}
\newcommand{\zv}{\phi}

\newcommand{\rv}{\gamma}
\newcommand{\rrv}{\mathbf{r}}
\newcommand{\Rv}{\Gamma}

\newcommand{\ov}{\mathbf{0}}

\newcommand{\sv}{\mathbf{s}}

\newcommand{\Xv}{\mathbf{X}}

\newcommand{\Zv}{\Phi}

\newcommand{\Nv}{\mathbf{N}}
\newcommand{\Vv}{\mathbf{V}}
\newcommand{\id}{\mathbf{I}}

\newcommand{\Rb}{\mathbb{R}}
\newcommand{\Eb}{\mathbb{E}}
\newcommand{\Pb}{\mathbb{P}}

\newcommand{\lv}{\mathbf{1}}

\newcommand{\betav}{\boldsymbol{\beta}}

\newcommand{\lambdav}{\boldsymbol{\lambda}}

\newcommand{\thetav}{\boldsymbol{\theta}}

\newcommand{\tcr}{\textcolor{black}}

\newcommand{\Dc}{\mathcal{D}}
\newcommand{\Ac}{\mathcal{A}}
\newcommand{\Tc}{\mathcal{T}}
\newcommand{\Nc}{\mathcal{N}}

\newcommand{\Cc}{\mathcal{C}}
\newcommand{\Xc}{\mathcal{X}}

\newcommand{\Zc}{\mathcal{Z}}

\newcommand{\Fc}{\mathcal{F}}

\newcommand{\Bc}{\mathcal{B}}
\newcommand{\Ec}{\mathcal{E}}

\newcommand{\Lc}{\mathcal{L}}

\newcommand{\E}{{\mathbb E}}

\newcommand{\Rd}{{\mathbb{R}^d}}

\newcommand{\TT}{\intercal}

\newtheorem{defn}{Definition}
\newtheorem{theorem}{Theorem}

\newtheorem{Lemma}{Lemma}
\newtheorem{Proposition}{Proposition}
\newtheorem{assump}{Assumption}

\DeclareMathOperator*{\minimize}{minimize\,}
\DeclareMathOperator*{\subjectto}{subject\,to\,}

\ifodd 1
\newcommand{\congr}[1]{{\color{blue}#1}}
\else
\newcommand{\congr}[1]{#}
\fi

\ifodd 0
\newcommand{\congc}[1]{{\color{red}(Cong: #1)}}
\else
\newcommand{\congc}[1]{}
\fi

\newenvironment{proof}[1][Proof]{{\it #1. } }{\ \rule{0.5em}{0.5em}}

\long\def\appaddress#1{
     \ifnum\statePaper=0
    {
        {\def\and{\unskip\enspace{\rm and}\enspace}%
        \def\And{\end{tabular}\hss \egroup \hskip 1in plus 2fil
                \hbox to 0pt\bgroup\hss \begin{tabular}[t]{c} }%
        \def\AND{\end{tabular}\hss\egroup \hfil\hfil\egroup
                \vskip 0.25in plus 1fil minus 0.125in
                \hbox to \linewidth\bgroup \hfil\hfil
                    \hbox to 0pt  \bgroup \hss \begin{tabular}[t]{c}}
        \def\ANDD{\end{tabular}\hss\egroup \hfil\hfil\egroup
                \vskip 0.25in plus 1fil minus 0.125in
                \hbox to \linewidth \bgroup \hfil\hfil
                    \hbox to 0pt  \bgroup \hss\begin{tabular}[t]{c}\bfseries}
            \hbox to \linewidth\bgroup \hfil\hfil
            \hbox to 0pt\bgroup\hss \begin{tabular}[t]{c} 
            Anonymous Institution
            \end{tabular}
            \hss\egroup
            \hfil\hfil\egroup}
    }
    \else
    {
        {\def\and{\unskip\enspace{\rm and}\enspace}%
        \def\And{\end{tabular}\hss \egroup \hskip 1in plus 2fil
                \hbox to 0pt\bgroup\hss \begin{tabular}[t]{c} }%
        \def\AND{\end{tabular}\hss\egroup \hfil\hfil\egroup
                \vskip 0.25in plus 1fil minus 0.125in
                \hbox to \linewidth\bgroup \hfil\hfil
                    \hbox to 0pt  \bgroup \hss \begin{tabular}[t]{c}}
        \def\ANDD{\end{tabular}\hss\egroup \hfil\hfil\egroup
                \vskip 0.25in plus 1fil minus 0.125in
                \hbox to \linewidth \bgroup \hfil\hfil
                    \hbox to 0pt  \bgroup \hss\begin{tabular}[t]{c}\bfseries}
            \hbox to \linewidth\bgroup \hfil\hfil
            \hbox to 0pt\bgroup\hss \begin{tabular}[t]{c} #1
                                \end{tabular}
        \hss\egroup
        \hfil\hfil\egroup}
    }
   \fi
}

\begin{document}

%

%

\twocolumn[

\aistatstitle{Stochastic Linear Contextual Bandits with Diverse Contexts}

\aistatsauthor{ Weiqiang Wu \And Jing Yang \And  Cong Shen }

\aistatsaddress{ London Stock Exchange \And  The Pennsylvania State University \And University of Virginia } ]



\begin{abstract}
 In this paper, we investigate the impact of context diversity on stochastic linear contextual bandits. As opposed to the previous view that contexts lead to more difficult bandit learning, we show that when the contexts are sufficiently diverse, the learner is able to utilize the information obtained during exploitation to shorten the exploration process, thus achieving reduced regret. We design the LinUCB-d algorithm, and propose a novel approach to analyze its regret performance. The main theoretical result is that under the diverse context assumption, the cumulative expected regret of LinUCB-d is bounded by a constant. As a by-product, our results improve the previous understanding of LinUCB and strengthen its performance guarantee. 
\end{abstract}

\section{Introduction}
In many applications, such as resource allocation in cloud computing platforms, or treatment selection for patients in clinical trials, the diverse user preferences and characteristics impose urgent need of personalized decision-making. In order to make optimal decisions for different individuals, the decision maker must learn a model to predict the reward when a decision is taken under different contexts. This problem is often formulated as a contextual bandit problem \citep{Auer:2003:UCB,Langford:2008}, which generalizes the classical multi-armed bandit (MAB) framework \citep{lai1985asymptotically,auer2002finite,bubeck2012regret,agrawal2012analysis,agrawal2013further}.

The inclusion of contextual information in the decision-making process introduces more uncertainty into the MAB framework and creates significant challenges for the learning problem. Part of the difficulty in contextual MAB comes from the increased problem dimension, as context is added as part of the unknown environment. Existing literature mostly focuses on developing bandit algorithms and providing theoretical analysis by treating the context as \emph{structureless} side information. The resulting models and algorithms are generic, but represent a ``worst case'' scenario since very little, if any, structure of the context is exploited. 


In many real-world applications, however, context often exhibits sufficient level of \emph{diversity}, which has been largely overlooked in the existing studies. For example, user profile is considered as the context in recommendation systems \citep{Li:2010:LinUCB}. When the system serves a large number of users, the group of user profiles is likely to be very diverse.  As another example, contextual MAB has been adopted in service placement of mobile edge computing, which utilizes the time of day and mobile user types as the context \citep{Chen2019}. 

It is not difficult to see that in these applications, the context arrival exhibits sufficient diversity that may be beneficial to the bandit algorithm design. Intuitively, the diverse contexts create opportunities for the learner to reduce the learning regret: when an arm is pulled frequently as the optimal arm for certain contexts, its parameters can be estimated accurately with the rewards obtained during exploitation. Therefore, the learner does not have to spend much time exploring it when the instantaneous context is not in favor for it, thus shortening the exploration stage and speeding up the convergence. 


In this paper, we demonstrate this optimistic view of context diversity by investigating the impact of diverse contextual information on the learning regret under the stochastic linear contextual bandits framework~\citep{Bastani2017ExploitingTN}. We show that, instead of considering the context as part of the ``uncontrollable'' environment and passively ``reacting'' to the incoming context, proactively interacting between context and arm exploration allows learning to transfer between different contexts and leads to much better overall performance. Specifically, we consider a set of $K$ arms, where
the parameter of each arm is represented by a $d$-dimensional vector unknown to the learner. When an arm $a$ is pulled under a context $c$, the obtained reward is the inner product of the parameter of the arm and a feature vector determined by arm $a$ and context $c$, corrupted by noise. The objective of the learner is to use the information contained in the observed rewards to decide an arm to pull in response to the instantaneous context. Assuming independent and identically distributed (i.i.d.) contexts, we aim to show that when the contexts are sufficiently diverse, the cumulative learning regret  in expectation can be bounded by a constant. 

The main contributions {of this paper} are three-fold.
	First,	 we formally introduce the concept of \emph{context diversity} into the stochastic linear contextual bandit framework and present a novel geometric interpretation. Such geometric interpretation provides an intuitive viewpoint to understand and analyze the impact of context diversity on the learning performance of stochastic linear contextual bandits. 
		
		Second, we propose an Upper Confidence Bound (UCB) type algorithm, termed as LinUCB-d, for the contextual bandit model. The new formulation of {LinUCB-d} {enables a unique approach} to characterize the impact of context diversity and achieve finite cumulative regret. The results also extend {the} existing understanding {of} LinUCB and strengthen its performance guarantee.  
		
	Third, we design a novel approach to analyze the performance of LinUCB-d. There are two distinct features of our approach: First, we relate the uncertainty in the estimated rewards with the solution to a constrained optimization problem, and leverage the optimality of the estimator to bound the corresponding frequency of bad decisions during the learning process. Second, we propose a frame-based approach to isolate the error events on a frame basis and make the regret tractable. 
		{These} techniques are novel and may find useful in other related settings. 

\section{Problem Formulation}\label{Problem Setting}
We consider a set of $K$ items (arms) denoted as $[K] = \{1,2,\ldots,K\}$. Assume for each $a \in [K]$ there is a fixed but unknown {parameter} vector $\thetav(a) \in \Rd$. At each time $t$, the learner observes a random context $c_t$, which is generated according to an unknown distribution.
Next, the learner decides to pull an arm $a_t \in [K]$ based on the information available. The incurred reward $y_t$ is given by
$ y_t =r(a_t,c_t)+\eta_t$,
where $\eta_t$ is a random noise, and $r(a_t,c_t)$ is a linear function of $\thetav({a_t})$ and the feature vector $\xv ({a_t},{c_t})\in \Rb^d$, i.e.,
\begin{align}\label{eqn:reward}
r(a_t,c_t):=\thetav^\TT( {a_t}) \xv ({a_t},{c_t}).  
\end{align}
Here we use $\xv^\TT$ to denote the transpose of vector $\xv$.

Let $\Cc$ be the set of all contexts. For any $a \in [K]$, define
\begin{align}
\Cc_a &:= \left\{ c \in \Cc \mid r (a, c)> r(b,c), \forall b \in [K]\backslash\{a\}\right\}, \\ 
\Xc_a &:= \left\{ \xv (a,c) \in \Rd \mid c \in \Cc_a \right\},
\end{align}
 i.e., $\Cc_a$ is the subset of contexts under which arm $a$ is the best arm rendering the maximum expected reward, and $\Xc_a$ is the set of feature vectors when arm $a$ is pulled under contexts in $\Cc_a$.
Let $\Fc_t = \sigma(c_1,a_1,y_1,\dots,c_{t-1},a_{t-1},y_{t-1},c_t,a_t)$
be the $\sigma$-field summarizing the information available just before $y_t$ is observed. 
We make the following assumptions throughout the paper. 

\begin{assump}\label{assump:bounded}
\begin{enumerate}[leftmargin=12pt,topsep=0pt, itemsep=0pt,parsep=0pt]
\item[1)] \textbf{{Bounded} parameters:} For any $a\in[K]$, $c\in \Cc$, we have $\|\thetav(a)\|_2\leq s$, $\|\xv(a,c)\|_2\leq l$.
\item[2)]  \textbf{Minimum reward gap:} For any $a,b\in[K]$, $b \neq a$, $c\in\Cc_a$, $ r(a,c) - r(b,c) \geq \Delta>0$.
\item[3)]  \textbf{Conditionally 1-subgaussian noise:} Given $\Fc_t$, $\eta_t$ is conditionally 1-subgaussian with $\Eb[\eta_t|\Fc_t]=0$, $\Eb[ \exp( \lambda \eta_t ) | \Fc_t ] \le \exp(\frac{\lambda^2}{2})$ for any $\lambda>0$.
\item[4)]  \textbf{Stochastic context arrivals:} In each time $t$, $c_t$ is drawn from the context set $\Cc$ in an {i.i.d.} fashion according to a distribution $\nu$.
\item[5)]  \textbf{Diverse contexts:} For any arm $a\in[K]$, $\lambda_{\min}( \Eb_{c\sim\nu}[\xv(a,c)\xv(a,c)^\TT|c\in \Cc_a])>0$,
where $\lambda_{\min}(A)$ is the smallest eigenvalue of $A$.  
\end{enumerate}
\end{assump} 
Assumption~1.1 ensures that the maximum regret at any step is bounded. Assumption~1.2 indicates that under a given context $c$, the optimal arm is strictly better than any other sub-optimal arms. 
Such a reward gap affects the convergence rate of the proposed algorithm (similar to the stochastic MAB setting).
Assumption~1.3 allows us to utilize the induced super-martingale to derive exponentially decaying tail bound of the estimation error. We note that all three assumptions on the bandit model are standard in the bandit literature or in the study of linear bandits.

Assumption~1.4 is a non-critical assumption made for ease of exposition. Essentially, what is required for the main results to hold is the ergodicity of the context arrival process, i.e., contexts lying in certain favorable subsets recur frequently in time. 

Assumption~1.5 is however critical for our main results to hold. It is equivalent to the condition that 
\begin{align}
\Pb[\lambda_{\min} (\Phi_a^\TT \Phi_a)>0]>0,\quad\forall a\in[K],\label{assump1.5}
\end{align}
where {$\Phi_a$} is a random matrix whose columns are $d$ feature vectors associated with arm $a$ and $d$ i.i.d. contexts drawn according to the \tcr{conditional} distribution of $c$ given $c\in\Cc_a$. 
It implies two conditions: First, all arms in $[K]$ {could be} optimal under certain contexts, i.e., $\Cc_a\neq \emptyset$. Second, for all contexts in favor of the same arm (i.e., contexts in $\Cc_a$), they are sufficiently diverse so that the corresponding feature vectors span $\Rb^d$. If the first condition does not hold, the arms that are strictly sub-optimal have to be explored sufficiently frequently in order to be distinguished from the optimal arms, thus an $O(\log T)$ regret is unavoidable in this situation. For the second condition, although it seems strict at first sight, it is actually quite reasonable in practice. This is because if a feature vector falls in $\Xc_a$, we would expect that feature vectors drawn from a small neighborhood of it fall in $\Xc_a$ as well. Since small perturbations of a vector can form a full-rank matrix, it is thus reasonable to assume $\mbox{span}(\Xc_a)=\Rb^d$.

\if{0}
Assumption~1.5 motivates the following definition to measure the diversity in the contexts.
\begin{defn}[Diversity of context]\label{defn:diversity}
Denote \tcr{$\Xc_a^{\otimes d}$} as the set of matrices whose columns are $d$ distinct feature vectors drawn from $\Xc_a$. Denote $\lambda_{\min}(\Xv)$ as the smallest eigenvalue of $\Xv$. Then, the diversity of context for arm $a$ is defined as  
$\lambda_0(\Xc_a):=\sup_{\Zv\in \Xc_a^d} \lambda_{\min}(\Zv^T\Zv)$,
and the diversity of context for all arms is defined as 
$\lambda_0:=\min_{a\in[K]}\lambda_0(\Xc_a)$.
\end{defn}

With the definition of context diversity in Definition~\ref{defn:diversity}, Assumption~1.5 is equivalent to the condition that $\lambda_0>0$. In the following, we denote the matrix $\Zv$ that achieves $\lambda_0(\Xc_a)$ as $\Zv_a$, and the context associated with the columns of $\Zv_a$ as $\bar{\Cc}_a$. 
\fi


Assume $\{\xv(a,c)\}_{a,c}$ is given and $\{\thetav(a)\}_a$ is unknown a priori. The cumulative regret of an online learning algorithm is defined as
\begin{align}\label{eqn:reg}
 R_T :=\sum_{t=1}^T r(a_t^*,c_t) - \sum_{t=1}^T y_t, 
\end{align}
where $a_t^*:=\arg\max_{a\in[K]} r(a,c_t)$.
While sublinear learning regret has been established for such linear contextual bandits~\citep{Abbasi:2011:IAL,Chu:SuperLin}, our objective is to investigate the fundamental impact of context diversity on the expected regret $\Eb[R_T]$. 


\section{Algorithm}

The existing linear contextual bandit algorithms such as the celebrated LinUCB~\citep{Li:2010:LinUCB} can be directly applied to the considered bandit problem. However, such approaches ignore the diversity in context arrivals and offer little insight to the understanding of diversity. In this section, we propose a Linear Upper Confidence Bound algorithm to manifest the impact of the diversity of context on the scaling of the learning regret. To distinguish it from LinUCB, we term it LinUCB-d. 

%



We label all contexts that have appeared in the order of their first appearances. We assume there are $n_t$ different contexts that have appeared before time $t$. With a slight abuse of notation, we denote the subset of those contexts as $\Cc_t$. Besides, we add $d$ dummy contexts, and denote the subset as $\Cc_0$. In the following, we use $c\in \{1,2,\ldots,n_t+d\}$ to index the contexts, while the first $n_t$ contexts are in $\Cc_t$ and the last $d$ are the added dummy ones. 
For the added {dummy} contexts, we assume the corresponding feature vector $\xv(a, {n_t+j})=l \ev_j$, $j=1,\ldots, d$, where $\ev_j\in\Rb^d$ is the unit vector whose $j$th entry is 1, and $l$ is the upper bound on $\|\xv(a,c)\|_2$.


Let $\lv\{\Ec\}$ be an indicator function that takes value one when $\Ec$ is true and zero otherwise. Define $N_t(a,c):=\sum_{\tau=1}^{t-1} \lv \{a_{\tau}=a, c_\tau=c\}$, i.e., the number of times that arm $a$ is pulled under context $c$ up to time $t$. 
Denote $S_t(a,c)$ as the cumulative reward of pulling arm $a$ under context $c$ right before time $t$, i.e. $S_t(a,c)=\sum_{\tau=1}^{t-1}  y_{\tau} \cdot\lv\{a_{\tau}=a, c_\tau=c\}$ for any $c\in\Cc$. We point out that for the dummy contexts, i.e., $c=n_t+1,\ldots,n_t+d$, $S_t(i,c) = 0$ at any time $t$ {since the dummy contexts never appear}. 


To simplify the notation, we let $\lv_d$ be the row vector with $d$ $1$'s, and $\ov_d$ be the row vector with $d$ $0$'s. We also introduce the following matrix(vector)-form notations:
\begin{align*}
\Xv_t(a) & := \left[ \xv(a,1), \ldots, \xv({a},n_t), l\ev_1,\ldots,l \ev_d \right], \\
\Nv_t(a)  &:= \diag[N_t(a,1), \ldots,N_t(a,n_t),\lv_d ], \\
\sv_t (a)  &:= [ S_t(a,1), \ldots, S_t(a,n_t), \ov_d ]^\TT, \\
\Vv_t(a) &:=\Xv_t(a)\Nv_t(a)\Xv_t^\TT(a). 
\end{align*}
Besides, we use $\Nv_t^{-1}(a)$ to denote the pseudo-inverse of $\Nv_t(a)$ obtained by flipping its {\it non-zero} entries, i.e.,
\begin{align*}
&\Nv_t^{-1}(a)\\
&=\diag \left[ \frac{\lv \{N_t(a,1)\hspace{-0.02in}>\hspace{-0.02in}0\}}{ N_t(a,1)}, \ldots, \frac{\lv \{N_t(a,n_t)\hspace{-0.02in}>\hspace{-0.02in}0\}}{ N_t(a,n_t)},\lv_d\right].
\end{align*}

The proposed LinUCB-d algorithm is presented in Algorithm~\ref{alg:linUCB-d}, where we set $f(t) : = 1 + t \log^2t $ in the expression of $\alpha_t$. It adopts the Optimism in Face of Uncertainty (OFU) principle where the learner always chooses the arm with the highest potential reward after padding a UCB term. 



\begin{algorithm}[t]
\caption{LinUCB-d}\label{alg:linUCB-d}
\begin{algorithmic}[1]
 \STATE \textbf{Initialization:} Set $\Nv_1(a)=\diag[\lv_d]$, $\sv_1(a)=\ov_d^\TT$ for all $a \in[K]$. 
\FOR{ $t = 1\ldots, T$ }
	\STATE Observe the incoming context $c_t$ and set 
	\begin{align*}
	\alpha_t =l s +\sqrt{(2+d)\log f(t)}.
	\end{align*}
	\vspace{-0.2in}
		\FOR { $a = 1,2,\ldots, K$ }
			\STATE Compute $\betav_t(a) = \Nv_t(a) \Xv_t^\TT(a) \Vv^{-1}_t(a)  \xv(a,{c_t}) $.
			\vspace{-0.1in}
			\STATE Set 
				\vspace{-0.15in}
\begin{align*}
\hat{r}_t(a) &= \sv_t^\TT(a) \Nv_t^{-1}(a) \betav_t(a), \\
\hat{\sigma}_t(a) &= \sqrt{\betav_t^\TT(a) \Nv_t^{-1}(a) \betav_t(a)}.
\end{align*}		
	\vspace{-0.25in}	
		\ENDFOR
		\STATE Pull arm $a_t = \arg\max_{a\in[K]}  \hat{r}_t(a) + \alpha_t\hat{\sigma}_t(a)$, and observe the reward $y_t$.	
	\STATE Update $\Xv_{t+1}(a_t)$, $\Nv_{t+1}(a_t)$, $\sv_{t+1}(a_t)$.
\ENDFOR
\end{algorithmic}
\end{algorithm}

We have a critical observation about LinUCB-d, as summarized in Proposition~\ref{prop:beta}, whose proof can be found in Appendix~\ref{appx:prop:beta}.

\begin{Proposition}\label{prop:beta}
$\betav_t(a)$ in Algorithm~\ref{alg:linUCB-d} is the solution to the following optimization problem: 
\begin{eqnarray}
&\minimize\limits_{\betav \in \mathbb{R}^{n_t+d}}  & \betav^\TT  \Nv^{-1}_t(a) \betav \nonumber \\
&\subjectto & \xv(a,{c_t}) =\Xv_t(a) \betav.
\label{eqn:beta}
\end{eqnarray}
\end{Proposition}

\textbf{Remark:} The rationale behind Algorithm~\ref{alg:linUCB-d} can be intuitively explained as follows: For each incoming $c_t$, the learner needs to estimate the expected reward for each of the arms before it decides which one to pull. Due to the linear reward structure in (\ref{eqn:reward}), if we are able to express $\xv(a,c_t)$ as a linear combination of the feature vectors in $\{\xv(a,c)\}_{c\in\Cc_t\cup\Cc_0}$ in the form of $\Xv_t(a)\betav$, then, the expected reward $r(a,c_t)$ can be expressed as $\rrv(a)\betav$, where $\rrv(a):=\thetav(a)^\TT\Xv(a)$. Since $\rrv(a)$ can be estimated based on observed rewards generated by pulling arm $a$ in the past, we can then estimate $r(a,c_t)$ directly without trying to estimate $\thetav(a)$ first.

Thus, the problem boils down to obtaining a valid representation of $\xv(a,c_t)$ in the form of $\Xv_t(a)\betav$. The existence of such a representation can be guaranteed by including the $d$ unit vectors associated with the dummy contexts in $\Xv_t(a)$. On the other hand, such a representation may not be unique when arm $a$ is pulled and more feature vectors are added to $\Xv_t(a)$. That is when Proposition~\ref{prop:beta} comes into play: by minimizing the objective function in (\ref{eqn:beta}) subject to the linear constraint, we pick the representation that minimizes the \tcr{uncertainty in} the estimated $r(a,c_t)$.

We point out that inclusion of the dummy contexts introduces bias to the estimation. However, as $t$ increases and $\Xv_t(a)$ gets expanded by including more feature vectors, the bias caused by the dummy contexts will vanish gradually. This is because under Assumptions~\ref{assump:bounded}.4 and \ref{assump:bounded}.5, the optimal solution to (\ref{eqn:beta}) will put more and more weights on feature vectors associated with the observed contexts instead of the dummy ones.

 
\if{0}
Proposition~\ref{prop:beta} indicates that the $\betav_t(a)$ involved in the estimated reward $\hat{r}_t(a)$ and its UCB is obtained by solving a constrained optimization problem.
The constraint ensures that the feature vector $\xv(a,c_t)$ is represented by a linear combination of $\{\xv(a,c)\}_{c\in\Cc_t\cup\Cc_0}$. The inclusion of the feature vectors associated with the dummy contexts in $\Xv_t(a)$ guarantees the existence of at least one of such linear combinations. 


Since $\xv(a,c_t)=\Xv_t(a) \betav$, we thus have $r(a,c_t)$ equal to $\thetav(a)^\TT \Xv_t(a) \betav_t(a)$. Even though $\thetav(a)$ is unknown, $\thetav^\TT(a)\xv(a,c)$ can be estimated based on the sample averages of the corresponding obtained rewards. The expected reward of pulling arm $a$ under $c_t$ can thus be estimated as a linear combination of those empirical averages, although the existence of the dummy contexts make such estimation biased. The corresponding estimation uncertainty is captured by $\hat{\sigma}_t(a)$. Thus, the solution to the optimization problem in (\ref{eqn:beta}) can be thought of as the linear combination that minimizes the estimation error.
\fi

Proposition~\ref{prop:beta} provides a brand new angle to view the linear contextual bandit problem. Leveraging this new viewpoint and the additional diversity assumption on the contexts, we will show that a constant regret can be achieved under LinUCB-d.

We note that LinUCB-d turns out to have deep connections with LinUCB. In order to avoid diversion from the main focus of this work, which is to elucidate the fundamental impact of context diversity on learning regret, we leave the comparison with LinUCB to Appendix~\ref{appx:prop:equivalence}.


\section{Analysis: Finite Contexts}\label{sec:uniform} 
In order to obtain some insights on how the diversity of context could help reducing the learning regret, in this section, we focus on a scenario where the context $c_t$ is drawn in an i.i.d. fashion from a finite set $\Cc$ according to a uniform distribution. With insight obtained for this scenario, we will extend the result and analysis to a general context distribution setting in Section~\ref{sec:general}.


According to Assumption~\ref{assump:bounded}.5, there must exist at least one subset of $d$ distinct contexts in $\Cc_a$, such that the corresponding feature vectors span $\Xc_a$.  Denote 
\begin{align*}
 \bar{\Phi}_a&:=\arg \max_{\Zv_a} \lambda_{\min}(\Zv^\TT_a\Zv_a), \\
 \lambda_0&:=\min_{a\in[K]}\lambda_{\min}(\bar{\Zv}^\TT_a\bar{\Zv}_a),
\end{align*}
and $\bar{\Cc}_a$ as the $d$ contexts associated with the feature vectors in $\bar{\Phi}_a$. Then, under Assumption~\ref{assump:bounded}.5, $\lambda_0>0$. Intuitively, $\lambda_0$ can be used as a metric for the diversity of context under this setting. 
We present our main theoretical result for the finite contexts setting as follows. 
\begin{theorem}\label{thm:main}
Under Assumption~\ref{assump:bounded}, if the context arrival $c_t$ is uniformly distributed over a finite set $\Cc$ with $|\Cc|=n$, the expected regret under Algorithm~\ref{alg:linUCB-d} can be bounded by $O\left(Kdn^2 +  \frac{dn(K+\delta^2)}{\Delta^2} \log\left(\frac{dn(K+\delta^2)}{\Delta^2}\right)\right)$, where $\delta=l\sqrt{d/\lambda_0}$.
\end{theorem}
Theorem~\ref{thm:main} indicates that the expected regret is bounded by a constant, which is in stark contrast to the state-of-the-art results on linear contextual bandits. It indicates that diverse contexts can indeed help to accelerate the learning process and make it converge to the optimal solution within finite steps on average. Besides, the constant bound monotonically decreases as $\lambda_0$ increases, which is consistent with our intuition that larger diversity of context is more advantageous in learning. 

We point out that the dependence on the number of contexts $n$ in the upper bound can be further reduced to a constant that does not scale in the total number of contexts, as we will show in the general context distribution setting in Section~\ref{sec:general}.




\subsection{Sketch of the Proof of Theorem~\ref{thm:main}} 
The complete proof of Theorem~\ref{thm:main} can be found in Appendix~\ref{appx:finite}. In this section, we provide a sketch of the proof to highlight the key ideas and shed light on the profound impact of context diversity to the learning performance.

The bounded regret in Theorem~\ref{thm:main} can be intuitively explained in this way: thanks to context diversity under Assumption~1.5, arms that are suboptimal for a given context are optimal for some other contexts. Since contexts are drawn in an i.i.d. fashion, then, with high probability, each arm will be played as an optimal arm for a linear fraction of time. Context diversity then ensures that for any arm $a$, the feature vector $\xv(a,c_t)$ for any incoming context $c_t$ can be expressed as a linear combination (denote the coefficient vector as $\bar{\betav}(a,c_t)$) of the columns of $\bar{\Zv}_a$. We note that $\{r(a,c)\}_{c\in\bar{\Cc}_a}$ can be estimated accurately based on the rewards collected when $a$ is pulled as an optimal arm. Hence, if $\bar{\Cc}_a$ were given a priori, the error of using the linear combination of $\{r(a,c)\}_{c\in\bar{\Cc}_a}$ to predict $r(a,c_t)$ would decrease in the order of {$O(1/\sqrt{t})$}. To overcome the difficulty that $\bar{\Cc}_a$ is unknown beforehand, LinUCB-d greedily selects the linear combination (with coefficient vector $\betav_t(a)$) to minimize the estimation uncertainty. Then, according to Proposition~\ref{prop:beta}, the corresponding estimation uncertainty must be lower than that if $\bar{\betav}(a,c_t)$ were used, leading to a faster decay of the prediction error. 

As explained above, the key to the result in Theorem~\ref{thm:main} is to show that each arm will be played as an optimal arm for a linear fraction of time. In order to show this, we propose a novel frame-based approach.

Specifically, we divide the time axis into frames with lengths $2^k$, $k=1,2,\ldots$, starting at $t=1$. Denote $F_k$ as the time slots lying in the $k$-th frame, i.e.,
$F_k:=\left\{t\mid 2^{k-1}\leq t\leq  \min(2^{k}-1, T)\right \}.$
Denote $N_t(c)$ as the number of times that context $c$ appears up to time $t$, and $N_{F_k}(c)$ as the number of times context $c$ appears in $F_k$, i.e., $N_{F_k}(c):=N_{2^k}(c) - N_{2^{k-1}}(c)$. Similarly, we define $N_{F_k}(a,c)$ as the number of times arm $a$ is pulled under context $c$ in $F_k$.
We consider the following error events:

\textbf{Irregular context arrivals.} For each arm $a\in[K]$, we focus on the $d$ contexts in $\bar{\Cc}_a$. Within a frame, if the total number of arrivals of any context $c\in\bar{\Cc}_a$ is smaller than half of its {\it expected} number of arrivals in that frame, we term it irregular context arrivals. If irregular context arrivals happen in frame $k$, we will put all time indices in the $(k+1)$th frame in $\Ac_T$, i.e.,
$\Ac_T :=\cup_{k} \left\{F_{k+1} \middle| \exists a,c\in\bar{\Cc}_a, \mbox{s.t. } N_{F_k}(c) \leq \frac{1}{2n}\cdot 2^{k-1} \right\}. $

Intuitively, due to the i.i.d. context arrival assumption, the probability of having irregular context arrivals in the $k$th frame decays {exponentially} in the length of frame $k$. Thus, the corresponding regret over $\Ac_T$ can be bounded by a constant. The detailed analysis can be found in Appendix~\ref{appx:At}.

\textbf{Bad estimates.} At time $t$, if the estimated reward $\hat{r}_t(a)$ deviates from its expected value $r(a,c_t)$ by more than $ \alpha_t\hat{\sigma}_t(a) $, we term it a bad estimate. We group the time slots with bad estimates over $(0,T]$ in $\Bc_T$, i.e.,
$\Bc_T := \left\{ t \mid \exists a\in[K], s.t. \left|\hat{r}_t(a) -r(a,c_t) \right| > \alpha_t\hat{\sigma}_t(a) \right\}. $
The regret over $\Bc_T$ can be bounded by a constant by adapting the Laplace method~\citep{LS19bandit-book} to our setting. The detailed analysis is deferred to Appendix~\ref{appx:Bt}.

\textbf{Bad presence of good estimates.} Within a frame, if the total number of time slots with bad estimates exceeds $\frac{1}{4n}$ of the frame length, we term the event bad presence of good estimates. If such an event happens in frame $k$, we put all time indices in the $(k+1)$th frame in $\Cc_T$, i.e.,
$\Cc_T := \cup_{k} \left\{F_{k+1} \middle|   | \Bc_T \cap F_k|  \geq \frac{1}{4n} \cdot 2^{k-1}\right\}, $
where $\left| \Bc_T \cap F_k \right|:=B_k$ is the number of bad estimates in frame $F_k$. As shown in Appendix~\ref{appx:Ct}, $|\Cc_T|$ can be upper bounded by a linear function of $|\Bc_T|$. The regret over $\Cc_T$ can thus be bounded as a linear function of the regret over $\Bc_T$.

\textbf{Pulling sub-optimal arms in good time slots.} For any time slot $t$ not included in $\Ac_T$, $\Bc_T$ or $\Cc_T$, we call it a good time slot. The learner may still pull a sub-optimal arm in a good time slot, due to the overlap of the confidence intervals of ${r}(a,c_t)$. We group the time slots when such event happens in $\Dc_T$. Specifically, 
$\Dc_T := \{t\mid t\notin \Ac_T \cup \Bc_T \cup \Cc_T, a_t\neq a^*_t \}.$

While the regrets over $\Ac_T$, $\Bc_T$ or $\Cc_T$ can be bounded in a relatively straightforward way, characterizing the regret over $\Dc_T$ relies on the context diversity, and is the most critical step towards the constant regret in Theorem~\ref{thm:main}. The detailed analysis is provided in Appendix~\ref{appx:Dt}. It involves the following major steps: 
\begin{enumerate}[leftmargin=12pt,topsep=0pt, itemsep=0pt,parsep=0pt]
\item[1)] Up to time $t$, the number of times that an arm $a\in[K]$ is chosen as a sub-optimal arm scales as $O(\log t)$ (Lemma~\ref{lemma_nu}). 
\item[2)] Based on the definition of $\Dc_T$, for any $t\in\Dc_T$, the number of times $a$ is pulled as an optimal arm before $t$ scales linearly in $t$ (Lemma~\ref{lemma_tau}).
\item[3)] Leveraging Proposition~\ref{prop:beta}, the prediction error thus decreases in $O( 1/\sqrt{t})$ (Lemma~\ref{lemma_beta}), which implies that $\Dc_T$ can only happen before a fixed time (Theorem~\ref{thm:Dt}).
\end{enumerate}

After assembling the regrets over $\Ac_T$, $\Bc_T$, $\Cc_T$ and $\Dc_T$ together, the result in Theorem~\ref{thm:main} can be obtained. 

\textbf{Remark:} We point out that the operation of LinUCB-d itself does not depend on frames. We introduce them for the purpose of analysis only. Besides, LinUCB-d does not require the knowledge of $\bar{\Cc}_a$, $\bar{\Phi}_a$ or the distribution of $c_t$. It can operate under general context arrival processes, even if Assumption~\ref{assump:bounded} does not hold.



\if{0}
Let $t^*$ be the last slot in $\Dc_T$ before $n$ and $\hat{\alpha} \leq \frac{\delta^2}{\frac{t}{16\lC} -  \frac{4 K \alpha^2}{\Delta^2}}$.
\begin{align}
LS+\sqrt{ 2\log f(t)+d\log (1+t)}\leq LS+\sqrt{ (2+d)\log f(t)} :=\alpha
\end{align}
\fi

\section{Analysis: General Context Arrivals}\label{sec:general} 
In this section, we extend the analysis for the finite uniform context distribution setting to the general context distribution setting. Compared with the finite contexts case, the major difference for the general setting is that the context set $\Cc$ could be infinite and even uncountable. Although LinUCB-d still works in the same way, the corresponding performance analysis becomes much more challenging. For the finite contexts case, we choose a set of feature vectors (columns in $\bar{\Zv}_a$) as the basis for $\Xc_a$, and show that a linear combination of the corresponding empirical average rewards leads to a fast decaying estimation error, as the number of times $a$ is pulled under contexts in $\bar{\Cc}_a$ scales linearly in time. However, for general context arrivals, the recurrence of any finite subset of contexts may have probability zero. Thus, the previous analysis cannot be extended straightfowardly to handle such case. 

To overcome such challenges, we make the following modifications: First, we extend the definition of $\bar{\Cc}_a$ from $d$ distinct contexts to $d$ non-overlapping {\it meta-contexts}, where each meta-context consists of a subset of contexts with a non-zero probability mass. Thus, the meta-contexts recur frequently, similar to the finite contexts setting. One subsequent challenge associated with the meta-contexts is that feature vectors associated with the contexts in a meta-context are different and occur randomly. Thus, we cannot fix a basis (such as the columns in $\bar{\Zv}_a$ as in the finite contexts case) beforehand for $\Xc_a$, as the corresponding contexts may not appear frequently in time. Rather, it needs to be adaptively selected based on up-to-date observations. How to ensure the existence of such a valid basis at each time is thus challenging. 



We construct the meta-contexts and a basis for each arm $a$ as follows. First, we select a matrix $\Phi_a$ with $\lambda_{\min}(\Phi_a^\TT\Phi_a)>0$, and denote its columns as $\{\zv_{a}^{(i)}\}_{i=1}^d$. Assumption~\ref{assump:bounded}.5 ensures the existence of such $\Phi_a$ for each $a\in[K]$ according to (\ref{assump1.5}). Let
\begin{align}\label{eqn:lambda0}
\lambda_0(\{\Phi_a\}):=\min_{a\in[K]} \lambda_{\min}(\Phi_a^\TT\Phi_a).
\end{align}
Then, we have $\lambda_0(\{\Phi_a\})>0$ with the selected $\Phi_a$s.

We then divide $\Xc_a$ into $d$ disjoint groups $\{\Xc_{a}^{(i)}\}_{i=1}^d$ based on their closeness to $\{\zv_{a}^{(i)}\}_{i=1}^d$, and break the tie arbitrarily, e.g.,
\begin{align}
\Xc_{a}^{(i)} =\Bigg\{ \xv\in \Xc_a \Bigg|  &  \frac{\xv^\TT \zv_{a}^{(i)}}{\|\zv_{a}^{(i)}\|_2} <  \frac{\xv^\TT \zv_{a}^{(j)}}{\|\zv_{a}^{(j)}\|_2} \mbox{ for } j<i,\nonumber \\  
& \frac{\xv^\TT \zv_{a}^{(i)}}{\|\zv_{a}^{(i)}\|_2} \leq  \frac{\xv^\TT \zv_{a}^{(j)}}{\|\zv_{a}^{(j)}\|_2} \mbox{ for } j>i \Bigg\}.
\end{align}

Let $r=\frac{1}{2}\sqrt{\frac{\lambda_0(\{\Phi_a\})}{d}}$,  and $B(\zv_{a}^{(i)},r)$ be an $\ell_2$ ball centered at $\zv_{a}^{(i)}$ with radius $r$. Let $\bar{\Xc}_{a}^{(i)}:=\Xc_{a}^{(i)}\cap B(\zv_{a}^{(i)},r)$. Then, as shown in Lemma~\ref{lemma:span3} in Appendix~\ref{appx:thm:continuous}, a valid basis for $\Xc_a$ can be formed if an arbitrary vector is picked from each of the subsets $\{\bar{\Xc}_{a}^{(i)}\}_{i=1}^d$. We then take the sample average of the previously observed feature vectors in $\bar{\Xc}_{a}^{(i)}$ (denoted as $\hat{\zv}_{a}^{(i)}$) as the corresponding basis vector.  

The definition of $\bar{\Xc}_{a}^{(i)}$ induces the definition of meta-contexts $\bar{\Cc}_{a}^{(i)}$ as follows: $$\bar{\Cc}_{a}^{(i)}:= \{c\in\Cc_a\mid \xv(a,c)\in\bar{\Xc}_{a}^{(i)} \}.$$
Let 
\begin{align} \label{eqn:defp}
p(\{\Phi_a\}):=\min_{a,i} \Pb[c_t\in \bar{\Cc}_{a}^{(i)}].
\end{align} 
Then, Assumption~\ref{assump:bounded}.5 ensures that there exists $\{\Phi_a\}$ such that $p(\{\Phi_a\})$ is bounded away from zero. For ease of exposition, in the following, we simply use $p$ to denote $p(\{\Phi_a\})$ without causing ambiguity.

Denote $N_{F_k}(\bar{\Cc}_{a}^{(i)})$ as the total number of times that the contexts in meta-context $\bar{\Cc}_{a}^{(i)}$ appear up to time $t$. We then keep the definitions of $\Bc_T$ and $\Dc_T$ the same as in the finite context set setting and modify the definition of $\Ac_T$ and $\Cc_T$ as follows:
\begin{align*}
\Ac_T&:=\cup_k \left\{F_{k+1}\mid \exists i,a, \mbox{ s.t. } N_{F_k}(\bar{\Cc}_{a}^{(i)})\leq \left(\frac{p}{2}\right)2^{k-1} \right\}, \\
\Cc_T&:=\cup_k  \left\{ |\Bc_T\cap F_k|\geq \left(\frac{p}{4}\right)2^{k-1} \right\}.
\end{align*}
 

Intuitively, the regret over $\Bc_T$ remains unchanged, while the regrets over $\Ac_T$ and $\Cc_T$ can be obtained through a straightforward extension of the previous results in the finite contexts case. The challenge of the analysis thus lies in the analysis of the regret over $\Dc_T$, whose major steps are listed as follows.
\begin{enumerate}[leftmargin=12pt,topsep=0pt, itemsep=0pt,parsep=0pt]
\item[1)] We first show that over $\Dc_T$, the number of times that arm $a$ is pulled as a sub-optimal arm under contexts in $\bar{\Cc}_{a}^{(i)}$, $b\neq a$, grows sublinearly in $t$ (Lemma~\ref{lemma:cont_freq}). Compared with Lemma~\ref{lemma_nu}, the random occurrences of the multiple contexts included in each meta-context incur an extra factor of $2d\log \frac{d+t}{d}$ in the upper bound. 
\item[2)] We then show that the total number of times that arm $a$ is pulled as the optimal arm under contexts in $\bar{\Cc}_{a}^{(i)}$ scales linearly in $t$ (Lemma~\ref{lemma:cont_fraction}). 
\item[3)] Since $\{\hat{\zv}_{a}^{(i)}\}$ is a valid basis, by leveraging Proposition~\ref{prop:beta}, we show that the estimation uncertainty under LinUCB-d decays in $O(1/\sqrt{t})$ (Lemma~\ref{lemma:cont_sigma}).
\end{enumerate}
Putting everything together, we have the following bounded regret for the general case. The detailed proof is provided in Appendix~\ref{appx:thm:continuous}.
\begin{theorem}\label{thm:continuous}
Under Assumption~\ref{assump:bounded}, the regret under Algorithm~\ref{alg:linUCB-d} is upper bounded by
$O\left(\frac{Kd}{p^2} +\frac{d(2\delta^2+Kd)}{\Delta^2p}\log^2 \left(\frac{d(2\delta^2+Kd)}{\Delta^2p}\right)\right)$ for any valid choice of $\{\Phi_a\}$, where $\delta:=l\sqrt{d/\lambda_0(\{\Phi_a\})}$, and $\lambda_0$ and $p$ are defined in Eqn.~(\ref{eqn:lambda0}) and (\ref{eqn:defp}), respectively. 
\end{theorem}

Theorem~\ref{thm:continuous} indicates that even for the general context distribution setting where the contexts are drawn from a continuous set, we are still able to obtain a constant regret bound. Compared with the result in Theorem~\ref{thm:main}, the scaling in terms of $d$ and $\delta$ is larger, due to the inclusion of multiple contexts in the meta-contexts.

\textbf{Remark:} Similar to the finite contexts case, ${\Phi}_a$, $\hat{\phi}_a^{(i)}$, $\bar{\Xc}_a^{(i)}$, $\bar{\Cc}_a^{(i)}$, and $p$ are introduced for the purpose of analysis only, and are not required for LinUCB-d.

\section{Experimental Evaluation}\label{sec:simulation}

\subsection{Uniform Context Arrivals} 
First, we consider a simplified scenario with 2 arms and 4 contexts for a proof of concept. We assume the arm parameters are
$\thetav(1) = (0.8, 0.4)$, $\thetav(2)  = (0.5, 0.7)$. The arm-context feature vectors are as follows: $\xv(1,1)=(0.9, 0.1)^\TT$, $\xv(1,2)=(0.75, 0.25)^\TT$, $\xv(1,3)=(0.25, 0.75)^\TT$, $\xv(1,4)=(0.1, 0.9)^\TT$, $\xv(2,1)=(0.8, 0.2)^\TT$, $\xv(2,2)=(0.7, 0.3)^\TT$,  $\xv(2,3)=(0.3, 0.7)^\TT$, $ \xv(2,4)=(0.2, 0.8) ^\TT$. The expected rewards for pulling the arms under the four different contexts can be calculated accordingly. Therefore, arm $1$ is the optimal arm under contexts $1$ and $2$ and arm $2$ is the optimal arm under contexts $3$ and $4$. We can
verify that $\{\xv(1,1), \xv(1,2)\}$ and $\{\xv(2,3), \xv(2,4)\}$ both span $\mathbb{R}^2$, thus they are valid basis for $\Xc_1$ and $\Xc_2$, respectively.

\if{0}
presented in Table~\ref{table:feature}.
\begin{table}[h]
\caption{Feature vectors $\xv(a,c)$}\label{table:feature}
\centering
\begin{tabular}{l|cc}
  \toprule[1.5pt]
$\xv(a,c)$  &  $a=1$  & $a=2$   \\
  \midrule
$c=1$  & $(0.9, 0.1)^\TT$ & $ (0.2, 0.8)^\TT$ \\
$c=2$ & $(0.75, 0.25)^\TT$   &  $(0.7, 0.3)^\TT$    \\
$c=3$& $(0.25, 0.75)^\TT$  & $(0.3, 0.7)^\TT$\\
$c=4$ &$(0.1, 0.9)  ^\TT$& $ (0.2, 0.8) ^\TT$\\
  \bottomrule[1.5pt]
\end{tabular}
\end{table}

The expected rewards for pulling the arms under the four different contexts are listed in Table~\ref{table:reward}.

\begin{table}[h]
\caption{Mean rewards $r(a,c)$}\label{table:reward}
\centering
\begin{tabular}{l|cccc}
  \toprule[1.5pt]
$r(a,c)$  &  $c=1$  & $c=2$ & $c=3$ & $c=4$   \\
  \midrule
$a=1$  & $0.76$ & $0.70$ & $0.50$ & $0.44$ \\
$a=2$ & $0.66$ & $ 0.56$ & $ 0.64$ & $0.66$    \\
  \bottomrule[1.5pt]
\end{tabular}
\end{table}
\fi

With the selected parameters, we 
compare LinUCB-d with the following baseline algorithms through simulation: 1) UCB with $\alpha_t=\sqrt{2\log f(t)}$ for individual contexts. We treat the arms under each context as a standard MAB and perform UCB for each context.
2) LinUCB with the same choice of $\alpha_t$ as in LinUCB-d.
3) A greedy LinUCB with $\alpha_t = 0$. This is the pure exploitation algorithm considered in \cite{Bastani2017ExploitingTN} essentially.

\begin{figure*}
	\centering  
	\subfigure[Uniform context arrivals.]{\includegraphics[width=0.32\textwidth]{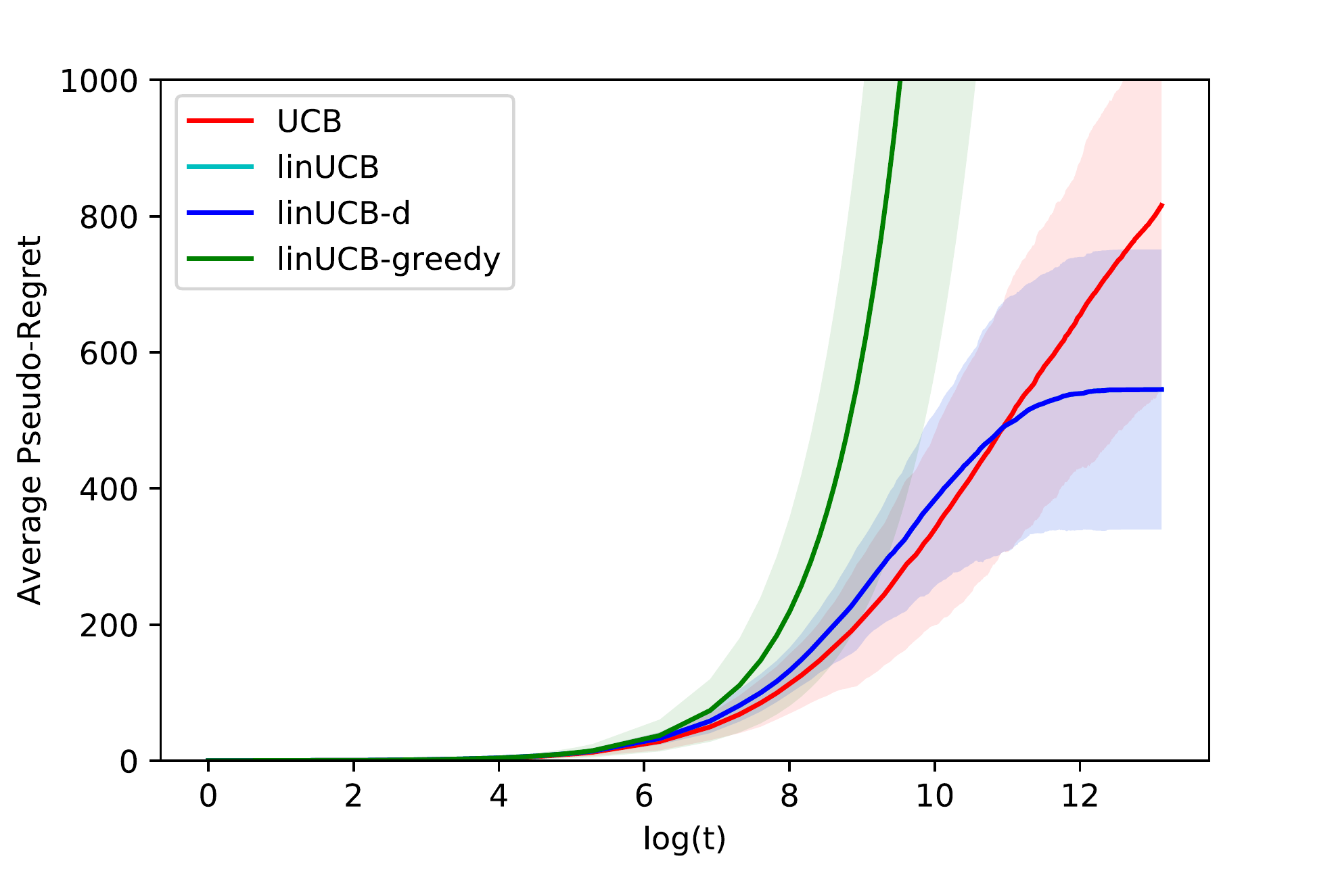}}
	\subfigure[Uniform context arrivals with different context diversity.]{\includegraphics[width=0.32\textwidth]{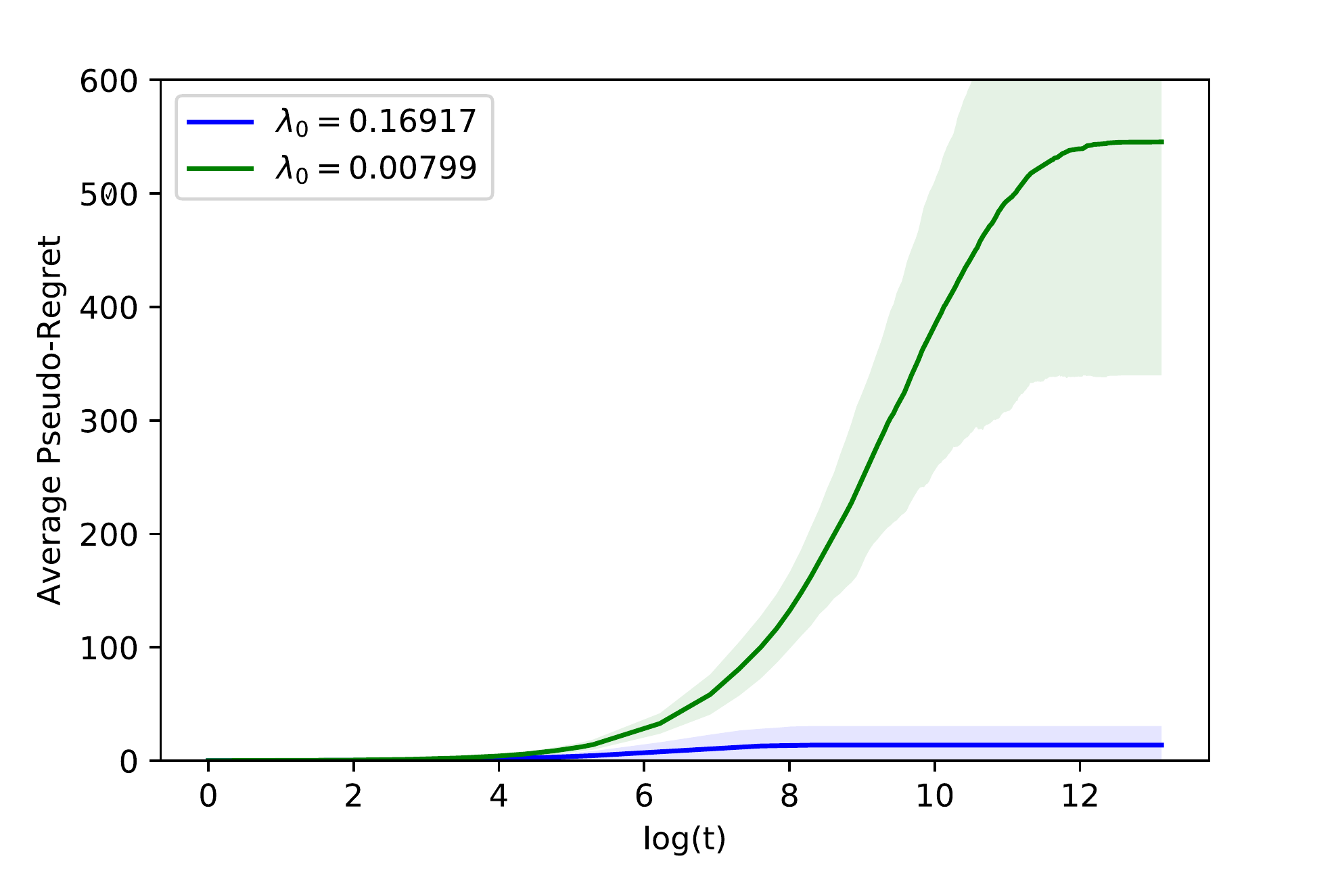}}
	\subfigure[General context arrivals.]{\includegraphics[width=0.32\textwidth]{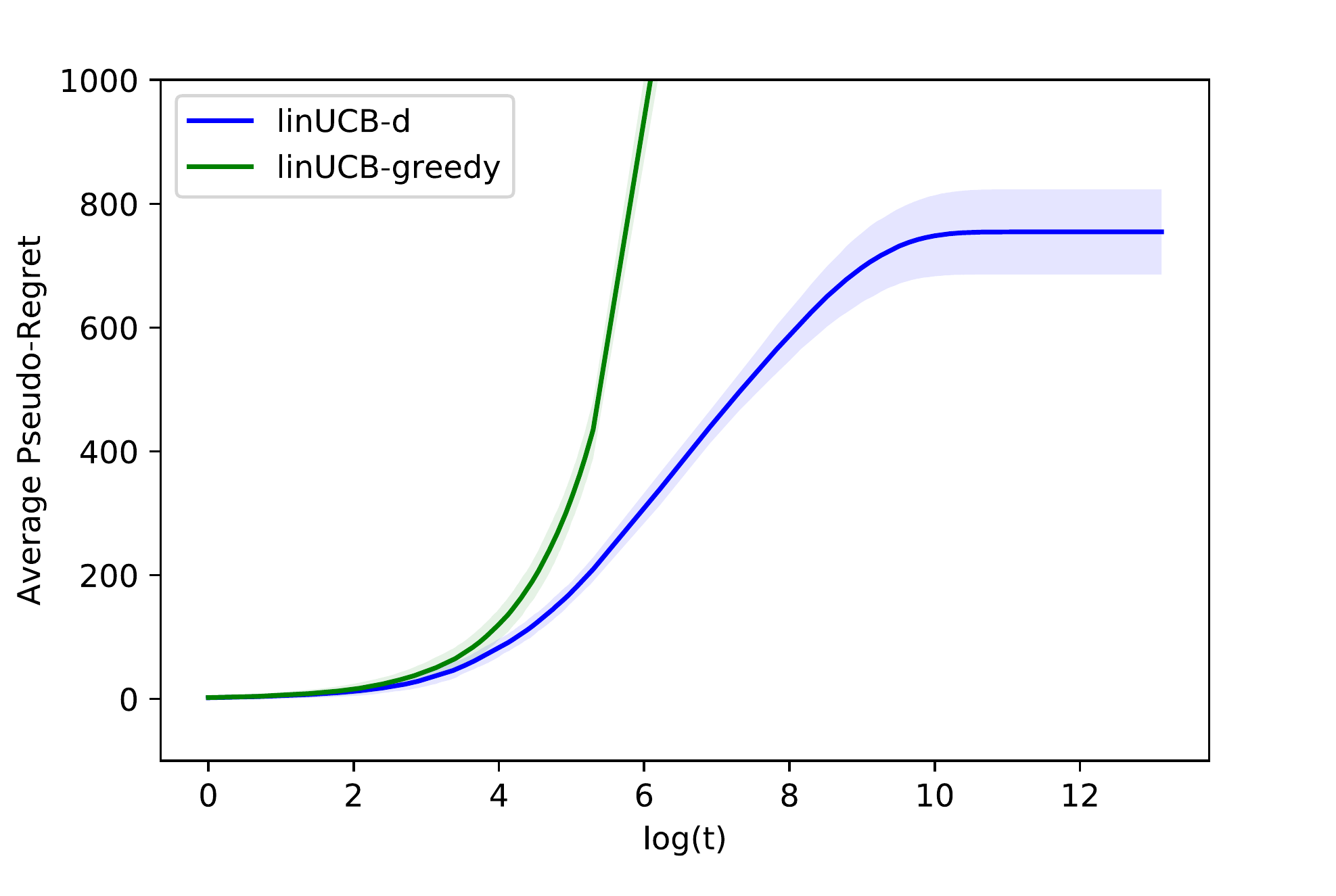}}
	\caption{Pseudo-regret over $\log T$. Shaded area indicates twice the standard deviation.}\label{fig:syn1}
\end{figure*}
\if{0}

\begin{figure*}[t]	\vspace{-0.1in}
	\centering  
	\subfigure[Pseudo-regret over $\log T$.]{\includegraphics[width=0.32\textwidth]{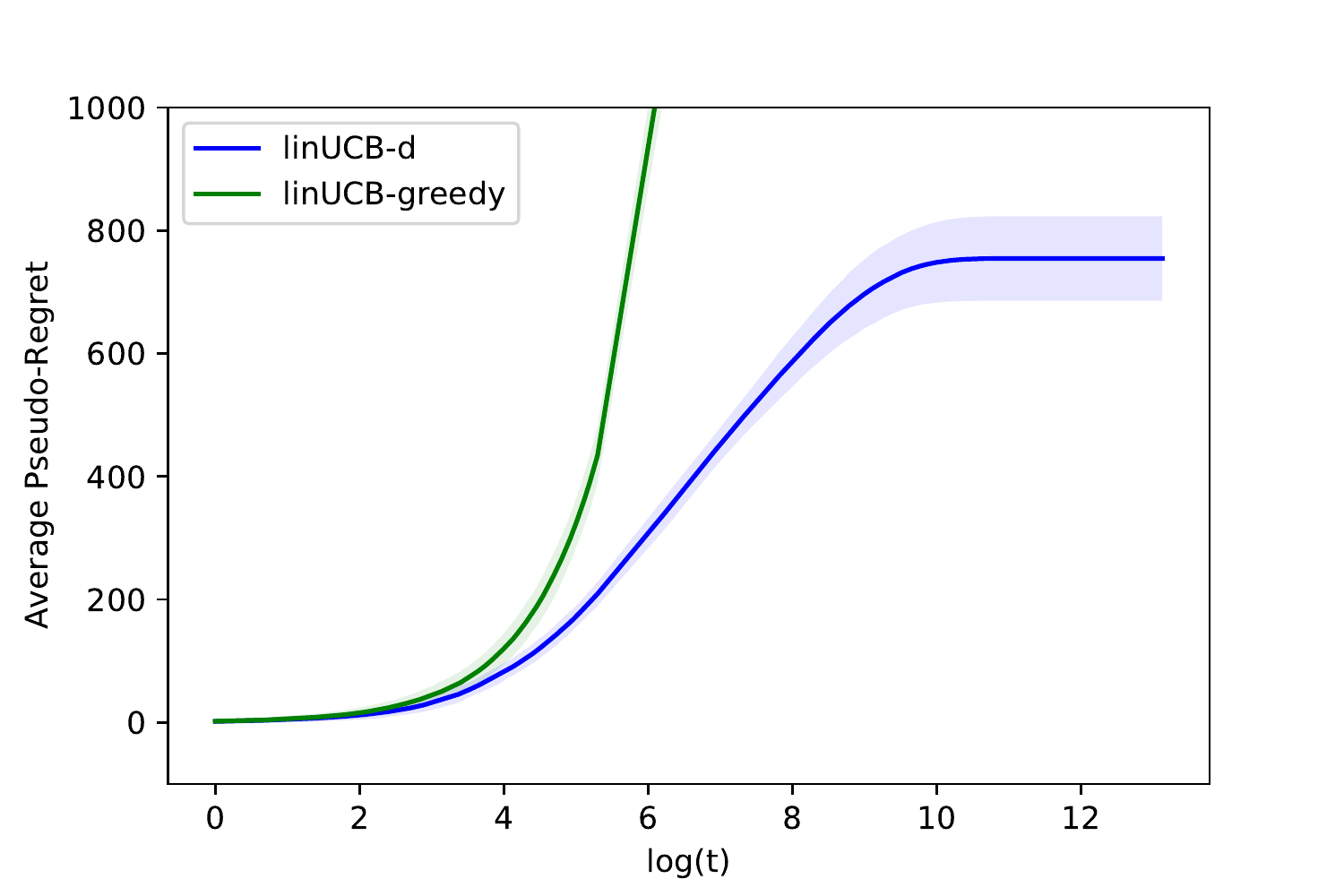}}
	\subfigure[Pseudo-regret over $T$.]{\includegraphics[width=0.32\textwidth]{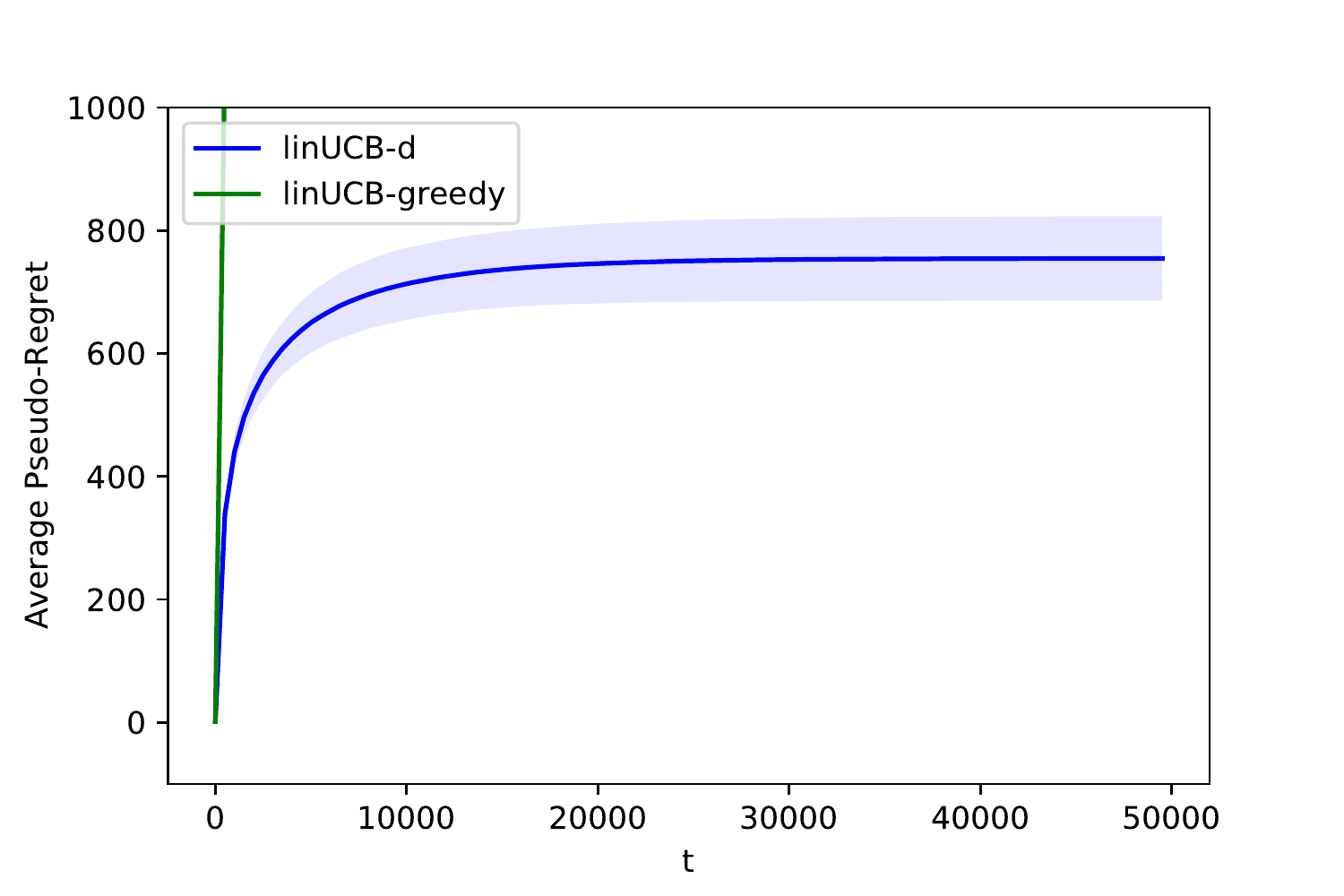}}
	\subfigure[Regret over $\log T$.]{\includegraphics[width=0.32\textwidth]{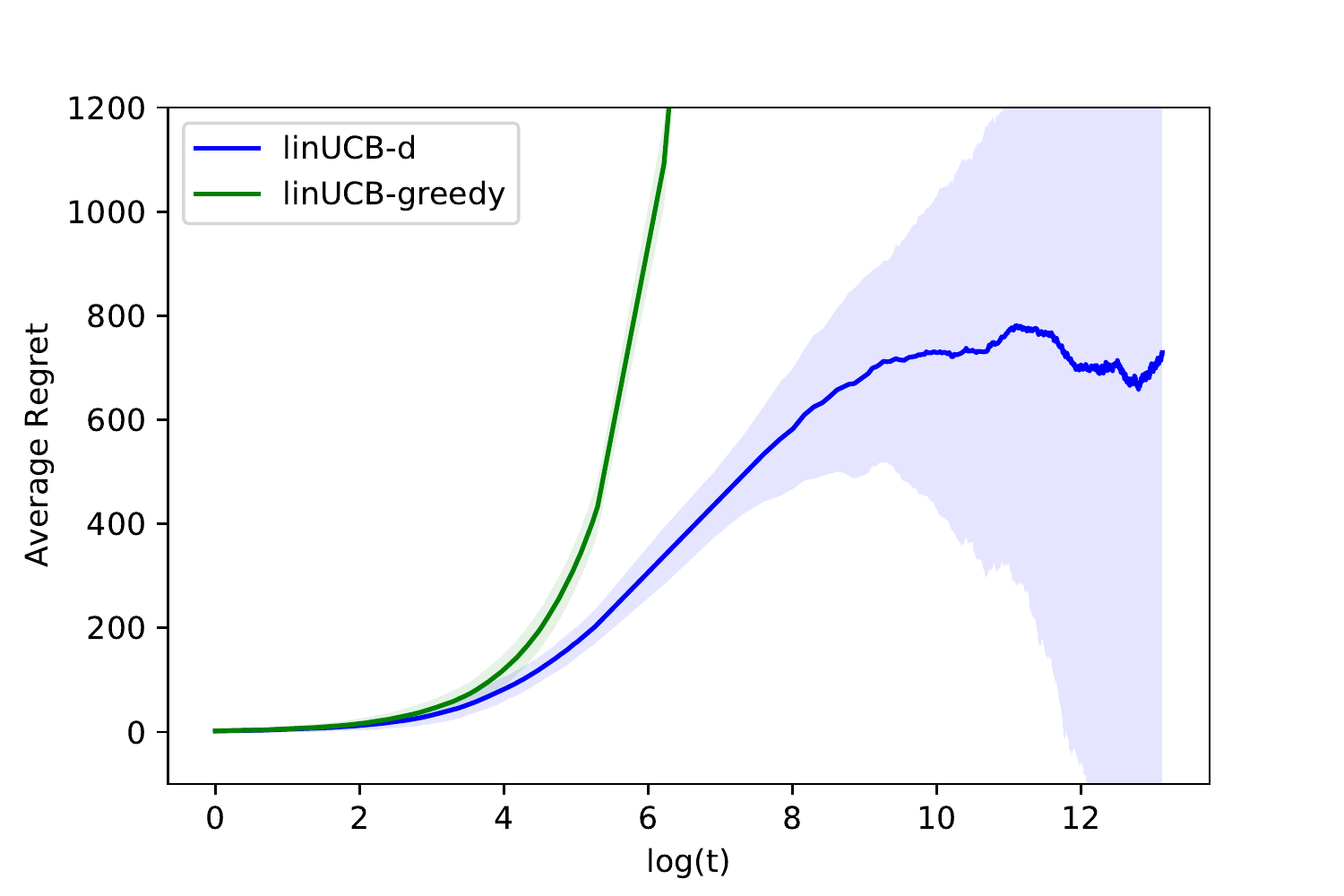}}
	\vspace{-0.1in}
	\caption{Regret versus time with ergodic context arrivals. Shaded area indicates twice the standard deviation.}\label{fig:syn2}
	\vspace{-0.15in}
\end{figure*}
\fi

For each algorithm, we randomly pick one out of those four contexts with probability $1/4$ each time, and add i.i.d. noise according to a standard Gaussian distribution $\Nc(0,1)$ to generate the reward. We run the simulation 100 times for each algorithm over 500,000 time slots. 
The sample average pseudo regrets are plotted in Fig.~\ref{fig:syn1}(a), where the pseudo regret is obtained by replacing $y_t$ in the definition of regret by $r(a_t,c_t)$, and the shaded area corresponds to twice of the standard deviation. As we expect, LinUCB-d with the same choice of $\alpha_t$ behaves exactly the same as LinUCB, and shows bounded regret. However, the greedy algorithm and UCB do not achieve constant regret. This indicates the following: First, the pure exploitation strategy does not work well in this case. This is because the selected parameters do not satisfy the {\it covariate diversity} defined in \cite{Bastani2017ExploitingTN}.
The covariate diversity in \cite{Bastani2017ExploitingTN} requires that the correlation matrix of the feature vectors lying in any half space is positive definite. It requires that there are feature vectors at least in any half space. Since the feature vectors in our example only lie in the first orthant, the covariate diversity condition is not satisfied and hence the greedy approach does not work well. Second, treating each context individually does not utilize the information obtained under other contexts about the same arm, thus cannot leverage the diversity of context to reduce the regret.

Next, we evaluate how $\lambda_0$ affects the regrets. We modify the feature vectors associated with contexts 2 and 3 while keeping the rest parameters the same. Specifically, we let $\xv(1,2)=(0.45, 0.65)^\TT$, $\xv(1,3)=(0.55, 0.35)^\TT$, $\xv(2,2)=(0.3, 0.5)^\TT$,  $\xv(2,3)=(0.7, 0.5)^\TT$. Compared with the previous setting, $\lambda_0$ increases from $0.00799$ to $0.16917$, while the reward gap $\Delta$ stays approximately the same. Intuitively, the basis vectors for each arm now point to more perpendicular directions and are more diverse in this sense. As indicated in Fig.~\ref{fig:syn1}(b), the increased diversity leads to much faster convergence and lower regret.

\subsection{General Context Arrivals} 
In this part, we investigate the performance of LinUCB-d with a more general context distribution. 
We first randomly generate parameter vectors in $\mathbb{R}^4$ for 5 arms under the constraint that $\|\thetav(a)\|_2= 10$. Thus, the arms are randomly located on a sphere in $\mathbb{R}^4$ with radius 10, which ensures that each of them can be optimal under certain contexts.   
For the feature vectors, we randomly draw $\xv(a, c_t) \in [0,1]^4$ for $a\in \{1, 2, 3, 4, 5\}$ at each time $t$ and make sure the reward gap condition in Assumption~1.2 is satisfied. We set $\Delta=0.5$ throughout the simulation. The contexts are drawn from a continuous set which includes infinite many contexts.

We only compare LinUCB-d with greedy LinUCB under this setup. This is because UCB for individual contexts cannot be run without recurring contexts, and LinUCB with the same $\alpha_t$ behaves the same as linUCB-d.
The sample average pseudo regrets are plotted in Fig.~\ref{fig:syn1}(c). 
As we observe, LinUCB-d still achieves constant regret, while the greedy algorithm does not converge.





	\section{Related Work}
	\label{sec:related}
	
	The model considered in this paper falls in the contextual bandits framework.
In the contextual MAB setting, the learner repeatedly takes one of $K$ actions in response to the observed context~\citep{Auer:2003:UCB}. Efficient exploration based on instantaneous context is of critical importance for contextual bandit algorithms to achieve small learning regret. The strongest known results~\citep{Auer:2003:UCB,Langford:2008,McMahan2009,beygelzimer11a,Dudk:2011,Agarwal2014} achieve an optimal regret after $T$ rounds of $O(\sqrt {KT})$ with high probability.

More specifically, our reward model is similar to that of linear contextual bandits in the literature. This setting is first introduced in~\cite{Auer:2003:UCB} through the LinRel algorithm and is subsequently improved through the OFUL algorithm in \cite{DaniHK08} and the LinUCB algorithm in \cite{Li:2010:LinUCB}. \cite{Tsitsiklis:2010:LinearBandits} extend the work of \cite{DaniHK08} by considering both optimistic and explore-then-commit strategies. It is shown in \cite{Abbasi:2011:IAL} that the regret can be upper bounded by $O(d\sqrt{T})$, where $d$ is the dimension of the context. A modified version of LinUCB, named SupLinUCB, is considered in \cite{Chu:SuperLin}, and shown to achieve $O(\sqrt{dT})$ regret.
Later, \cite{Valko:2013:Kernel} mix LinUCB and SupLinUCB with kernel functions and propose an algorithm to further reduce the regret to $O(\sqrt{\tilde{d}T})$, where $\tilde{d}$ is the effective dimension of the kernel feature space. This line of literature typically allows for arbitrary (adversarial) context sequences, and the $O(\sqrt {T})$ regret persists.

Recently, a few works start to take the diversity in contexts into consideration. \cite{goldenshluger2013} introduce a notion of diversity similar to Assumption~\ref{assump:bounded}.5 to a two-armed linear bandits setting. They show that the regret scales in $O(\log T)$ when a margin condition is satisfied, where the contribution from the ``large-margin'' covariates scales in $O(\log T)$. \cite{Bastani:2015} generalize the notation to a so called ``compatibility condition'' in a contextual linear bandits model with high-dimensional covariates, and investigate a LASSO based approach. They show that the regret can be bounded by a polynomial of $\log T$ under the margin condition. The $O(\log T)$ regret persists for error events associated with large-margin covariate vectors. In contrast, we show that a {\it bounded} regret can be achieved, by leveraging the geometric interpretation of the diversity condition and the reward gap condition. 

\cite{Bastani2017ExploitingTN} propose a concept called {\it covariate diversity}, which requires that the correlation matrix of the covariate vectors lying in any half space is positive definite. Under this condition, it shows that the exploration-free greedy algorithm is near-optimal for a two-armed bandit under the stochastic setting and achieves regret in $O(\log T)$. A perturbed adversarial setting with a similar notion of diversity is studied in \cite{kannan2018smoothed}. It shows that greedy algorithms can achieve regrets in $O(\sqrt{dT})$. We note that such condition is stronger than Assumption~\ref{assump:bounded}.5. As illustrated through {simulations} in Section~\ref{sec:simulation}, a greedy strategy may not work well under our setting, due to the difference between {the} diversity definitions.

\section{Conclusions}
The main purpose of this paper was to study the impact of \emph{context diversity} on the learning performance in stochastic linear contextual bandits. We have shown that, by adding an assumption that the context arrivals satisfy some diversity conditions, it is possible to significantly reduce the learning regret of contextual bandits. We proposed an algorithm called LinUCB-d and showed that when the diversity assumption is satisfied, the expected regret can in fact be upper bounded by a constant. This study illustrates the power of incorporating structure in the contexts to the bandit problem. It is of interest to evaluate whether other structures of the context can be similarly considered, and what their impacts would be. Another interesting problem is to study the impact of context diversity in other settings, such as the perturbed adversarial setting~\citep{kannan2018smoothed}. 



\subsubsection*{Acknowledgements}
    JY acknowledges the support from U.S. National Science Foundation under Grant ECCS-1650299.
    
	\bibliographystyle{apalike}
\bibliography{BanditYang}

\newpage
\onecolumn
\appendix

\thispagestyle{empty}
	
	\hsize\textwidth
    \linewidth\hsize \toptitlebar {\centering
    {\Large\bfseries Supplementary Material: Stochastic Linear Contextual Bandits with Diverse Contexts \par}}
    \bottomtitlebar

\aistatsauthor{ Weiqiang Wu \And Jing Yang \And  Cong Shen }

\appaddress{ London Stock Exchange \And  The Pennsylvania State University \And University of Virginia } 

    \vspace{0.2in}
\section{Proof of Proposition~\ref{prop:beta}}\label{appx:prop:beta}
For the constrained convex optimization problem in (\ref{eqn:beta}), the corresponding Lagrangian can be formulated as
\begin{align}
\Lc(\betav,\lambda)&=  \betav^\TT  \Nv^{-1}_t(a) \betav +\left(\xv(a,{c_t})-\Xv_t(a) \betav\right)^\TT \lambdav,
\end{align}
where $\lambdav \in \Rd$ is the Lagrangian multiplier vector.

Taking derivative with respect to $\betav$, we have
\begin{align}\label{eqn:kkt}
\Nv^{-1}_t(a) \betav -\Xv_t^\TT(a) \lambdav=0,
\end{align}
i.e.,
\begin{align}\label{beta}
\betav = \Nv_t(a) \Xv_t^\TT(a)  \lambdav.
\end{align}

Then, to satisfy the first constraint in (\ref{eqn:beta}), we have
\begin{align}\label{C}
\xv(a,{c_t}) = \Xv_t(a) \betav = \Xv_t(a) \Nv_t(a) {\Xv_t^\TT(a)} \lambdav,
\end{align}
which implies that
\begin{align}\label{lambda}
\lambdav = \left[ \Xv_t(a) \Nv_t(a) {\Xv_t^\TT(a)} \right]^{-1} \xv(a,{c_t}):= \Vv^{-1}_t(a)\xv(a,{c_t}).
\end{align}
Note that the definitions of $ \Xv_t(a)$ and $\Nv_t(a) $ ensure that $\Vv_t(a)$ is positive definite and invertible for every $t$. 
 
Plugging (\ref{lambda}) into (\ref{beta}), we have $$\betav = \Nv_t(a) \Xv_t^\TT(a) \Vv^{-1}_t(a)\xv(a,{c_t}),$$ which is the unique optimal solution to the optimization problem in (\ref{eqn:beta}).

%
%
%


\section{One the Relationship between LinUCB-d and LinUCB}\label{appx:prop:equivalence}



\textbf{Major difference.} One major difference between LinUCB-d and LinUCB~\citep{Li:2010:LinUCB} is as follows: Under LinUCB, at each time $t$, it will first estimate the true parameter of arm $a$ (i.e., $\thetav_a$) by solving a ridge regression and then use it to derive the UCB for the expected reward. The criterion to select the estimate is to minimize the penalized mean squared error in fitting the past observations; On the other hand, under LinUCB-d, the learner will directly estimate the expected reward through a linear combination of the rewards obtained when arm $a$ was pulled under all contexts. The criterion of selecting the estimate is to \tcr{minimize the uncertainty (or ``variance") of the estimation.} It avoids the intermediate step of trying to estimate $\thetav_a$ first in LinUCB. 


\textbf{Essential equivalence.} Although linUCB-d and linUCB view the problem from different angles, they actually produce the same estimate on the expected reward and confidence bound at every time $t$ under the same realizations of context arrivals and rewards, as shown below.

Based on (\ref{beta}), (\ref{C}) and (\ref{lambda}), the estimated mean reward $\hat{r}_t(a)$ in Algorithm~\ref{alg:linUCB-d} can be alternatively expressed as
\begin{align*}
\hat{r}_t(a)&  =  \sv_t^\TT(a) {\Xv_t^\TT(a)} \lambdav =\sv_t^\TT(a) {\Xv_t^\TT(a)} \Vv^{-1}_t(a) \xv(a,{c_t})  :=\hat{\thetav}^{\TT}_t(a) \xv(a,{c_t}), 
\end{align*}
where $\hat{\thetav}_t(a):= \Vv^{-1}_t(a) \Xv_t(a)\sv_t(a)$. We can verify that this is exactly the estimate of $\thetav(a)$ obtained by applying the ridge regression with penalty factor $l^2$ to the historical data $\{(\xv(a_\tau, c_\tau), y_\tau)\}_{\tau=1}^{t-1}$. 

Besides, for the $\hat{\sigma}_t(a)$ in Algorithm~\ref{alg:linUCB-d}, we have
\begin{align}
\hat{\sigma}_t(a)&=\sqrt{ \lambdav^\TT \Xv_t(a) \Nv_t(a) {\Xv_t^\TT(a)}\lambdav}= \|\xv(a,{c_t})\|_{\Vv^{-1}_t(a)}, \label{eqn:sigma}
\end{align}
where we follow the convention to denote $\xv\Vv\xv^\TT$ as $\|\xv\|^2_{\Vv}$.

Thus, if we let $l=1$, both $\hat{r}_t(a)$ and $\hat{\sigma}_t(a)$ share the same form as the corresponding quantities in LinUCB. As a {\it reformulation} of LinUCB, LinUCB-d automatically inherits all properties of LinUCB.

\textbf{Computation and analytical issues.} Computationally LinUCB-d is the same as LinUCB if we first compute the Lagrangian multiplier in (\ref{lambda}) through $\Vv_t(a)$, which can be equivalently computed by summing $\xv(a,c_t)\xv^\TT(a,c_t)$ over the time slots when $a$ is pulled. The advantage of LinUCB-d as an alternative form of LinUCB is on the analytical side. The prediction uncertainty minimization nature shown in Proposition~\ref{prop:beta} gives us a unique angle to elucidate the impact of context diversity on the corresponding learning regret, as elaborated in Lemma~\ref{lemma_nu}, Lemma~\ref{lemma_tau}, Lemma~\ref{lemma:cont_freq} and Lemma~\ref{lemma:cont_fraction}.

\section{Proof of Theorem~\ref{thm:main}}\label{appx:finite}
In the following, we will derive regret bounds for those error events individually, and then assemble them together to obtain the regret bound in Theorem~\ref{thm:main}.

\subsection{Bound the Regret over $\Ac_T$}\label{appx:At}
First, based on Hoeffding's inequality, and the {independent and uniform} arrival of contexts assumption, we have
\begin{align}
&\Pb \left[ N_{F_k}(c) \leq \frac{1}{2 n}\cdot 2^{k-1} \right] \leq \exp\left( -\frac{2^{k-1}}{2n^2} \right) .\label{eqn:irregular}
\end{align}
 Denote $ R(\Ac_T)$ as the regret incurred over $\Ac_T$, and $M$ as the maximum per-step regret. Then,
 \allowdisplaybreaks
\begin{align} 
 \Eb[R(\Ac_T)] &  \leq M\sum_{k=2}^{\lceil\log_2 T\rceil} \sum_{t \in F_k} \Eb[\lv\{t\in\Ac_T\}]  \leq M\sum_{k=1}^{\lfloor\log_2 T\rfloor} \sum_{t \in F_{k+1}} \sum_{a,c\in\bar{\Cc}_a} \Pb\left[ N_{F_{k}}(c) \leq \frac{2^{k-1}}{2 n} \right]\nonumber \\
 &	\leq MKd \sum_{k=1}^{\lfloor\log_2 T\rfloor} \sum_{t \in F_{k+1}} \exp \left( -\frac{2^{k-1}}{2n^2} \right) \leq Mn \sum_{t=2}^{\infty} \exp \left( -\frac{t}{8n^2} \right)  \label{eqn:irregular2}\\
 &\leq M Kd \int_0^\infty \exp \left( -\frac{t}{8n^2} \right) \diff t 
  = 8MKd n^2,\label{eqn:irregular3}
\end{align}
where (\ref{eqn:irregular2}) follows from (\ref{eqn:irregular}).

 
\subsection{Bound the Regret over $\Bc_T$}\label{appx:Bt}
First, we define
$\tilde{\sv}_t(a)$ and $\bar{\sv}_t(a)$ as follows:
\begin{align}
\tilde{\sv}_t(a)&= [ S_t(a,1)-N_t(a,1) r(a,1),\ldots, S_t(a,n_t)-N_t(a,n_t) r(a,n_t) ,\ov_d ]^\TT\\
\bar{\sv}_t(a)& = [ \ov_{n_t},-l\ev_1^\TT\thetav(a), \ldots, - l\ev_d^\TT \thetav(a) ]^\TT.
\end{align}
Intuitively, $\tilde{\sv}_t(a)$ corresponds to the accumulated noise in the observations when arm $a$ is pulled under different contexts up to time $t$, and $ \bar{\sv}_t(a)$ corresponds to the bias contributed by the feature vectors associated with the dummy contexts, which were added to ensure the existence of the unique solution in (\ref{eqn:beta}) for every $t$.

Then, the reward estimation error can be expressed as
\begin{align}
\hat{r}_t(a)-r(a,c_t)
&=\sv_t^\TT(a) \Nv_t^{-1}(a) \betav_t(a)- \thetav^\TT(a) \xv(a,c_t) \nonumber\\
&=\sv_t^\TT(a) \Nv_t^{-1}(a) \betav_t(a)-\thetav^\TT(a) \Xv_t(a) \betav_t(a)\label{eqn:error2-1}\\
&= \left( \sv_t(a) -\Nv_t(a) \Xv^\TT (a) \thetav(a)\right)^\TT\Nv_t^{-1}(a) \betav_t(a) \nonumber\\
&:=(\tilde{\sv}_t(a) +\bar{\sv}_t(a) )^\TT \Xv_t^\TT(a) \lambdav,\label{eqn:error2}
\end{align}
where (\ref{eqn:error2-1}) is due to the fact that $ \xv(a,c_t)=\Xv_t(a) \betav_t(a)$ according to Proposition~\ref{prop:beta}, and the $\lambdav$ in (\ref{eqn:error2}) is the Lagrangian multiplier involved in the proof of Proposition~\ref{prop:beta} in Appendix~\ref{appx:prop:beta} and satisfies (\ref{eqn:kkt}).
In the following, we will bound the contribution from $\bar{\sv}_t(a) $ and $\tilde{\sv}_t(a) $ in the estimation error, respectively.

We note that at any time $t$, $N_t(a,c)=1$ for $c=n_t+1,\ldots,n_t+d$. Besides, according to (\ref{eqn:sigma}) in Appendix~\ref{appx:prop:equivalence},
\begin{align*}
\hat{\sigma}_t(a)&= \sqrt{\lambdav^\TT \Xv_t(a) \Nv_t(a) {\Xv_t^\TT(a)}\lambdav} = \|\Nv_t^{1/2}(a) \Xv_t^\TT(a)\lambdav\|_2.
\end{align*}
Thus,
\begin{align}
|\bar{\sv}_t^\TT(a) \Xv_t^\TT(a) \lambdav| &= |\bar{\sv}_t^\TT(a) \Nv_t^{1/2}(a) \Xv_t^\TT(a) \lambdav|\\
&\leq \|\bar{\sv}_t^\TT(a)\|_2 \cdot  \| \Nv_t^{1/2}(a)  \Xv_t^\TT(a) \lambdav\|_2\label{eqn:bar_ss}\\
&=  \|\thetav(a)\|_2 \hat{\sigma}_t(a)\leq l s \hat{\sigma}_t(a)\label{eqn:bar_s},
\end{align}
where (\ref{eqn:bar_ss}) follows from the Cauchy-Schwarz inequality.

Before we proceed to bound $\tilde{\sv}_t^\TT(a) \Xv_t^\TT(a) \lambdav$, we first introduce the following notations. Recall that $\Vv_t(a):=\Xv_t(a) \Nv_t(a) \Xv_t^{\TT}(a)$. Let $\Vv_t^{1/2}(a)$ be its square root, i.e., $\Vv_t^{1/2}(a) \Vv_t^{1/2}(a) = \Vv_t(a) $. Let $\tilde{\Vv}_t(a):=\sum_{c=1}^{n} N_t(a,c) \xv(a,c) \xv^\TT(a,c)$, $\Vv_0:=l^2 \id$. Then, $\Vv_t(a)=\tilde{\Vv}_t(a)+\Vv_0$. We have
\begin{align}
|\tilde{\sv}_t^\TT(a) \Xv_t^\TT(a) \lambdav| &= |\tilde{\sv}_t^\TT(a) \Xv_t^\TT(a) \Vv_t^{-1/2}(a) \Vv_t^{1/2}(a) \lambdav|\nonumber\\
&\leq \|\tilde{\sv}_t^\TT(a) \Xv_t^\TT(a) \|_{\Vv_t^{-1}(a) }\|\Vv_t^{1/2}(a) \lambdav\|_2\\
& =\|\tilde{\sv}_t^\TT(a) \Xv_t^\TT(a) \|_{\Vv_t^{-1}(a) }\hat{\sigma}_t(a).\label{eqn:tilde_s}
\end{align}

We then adopt the Laplace method~\citep{LS19bandit-book} to bound $\|\tilde{\sv}_t^\TT(a) \Xv_t^\TT(a)  \|_{\Vv_t^{-1}(a) }$ as follows.

\begin{Lemma}\label{lemma:martingale}
Denote $M_t(\uv) : =\exp \left(\tilde{\sv}_t^\TT(a) \Xv_t^\TT(a)\uv - \frac{1}{2}\uv^\TT\tilde{\Vv}_t(a) \uv \right)$ for any $\uv$. Let $h(\uv)$ be a probability measure over $\Rb^d$. Then,  $\bar{M}_t:=\Eb_h[M_t(\uv)] $ is a super martingale with $\bar{M}_0=1$.
\end{Lemma}
\begin{proof}
First, we note that $M_t(\uv)-M_{t-1}(\uv)$ equals zero if arm $a$ is not pulled at time $t$. Then, for any fixed $\uv$, we have
\begin{align}
&\Eb[M_t(\uv)\mid \Fc_{t-1}] \nonumber\\
& = \Eb\left[\exp \left(\tilde{\sv}_t^\TT(a) \Xv_t^\TT(a)\uv - \frac{1}{2}\uv^\TT\tilde{\Vv}_t(a) \uv \right) \middle|\Fc_{t-1}\right]\\
&= \Eb\left[\exp \left(\left(\tilde{\sv}_t^\TT(a)-\tilde{\sv}_{t-1}^\TT(a)\right) \Xv_t^\TT(a)\uv - \frac{1}{2}\uv^\TT\left(\tilde{\Vv}_t(a)-\tilde{\Vv}_{t-1}(a)\right) \uv \right)\right]M_{t-1}(\uv) \label{eqn:martingale}
\end{align}

Based on the definition of $\tilde{\sv}_t$ in (\ref{eqn:tilde_s}), $\tilde{\sv}_t^\TT(a)-\tilde{\sv}_{t-1}^\TT(a)$ equals $\eta_t \ev^\TT_c$ if $a_t=a, c_t=c$, and zero otherwise; Similarly, for  $\tilde{\Vv}_t(a)-\tilde{\Vv}_{t-1}(a)$, it equals $\xv(a,c)\xv^{\TT}(a,c)$ if $a_t=t, c_t=c$, and zero otherwise. Therefore, if $a_t\neq a$, 
\begin{align}
\Eb\left[\exp \left(\left(\tilde{\sv}_t^\TT(a)-\tilde{\sv}_{t-1}^\TT(a)\right) \Xv^\TT(a)\uv - \frac{1}{2}\uv^\TT\left(\tilde{\Vv}_t(a)-\tilde{\Vv}_{t-1}(a)\right) \uv \right) \middle| a_t\neq a\right]=1.\label{eqn:equal0}
\end{align}
If $a_t= a$, $c_t=c$,
\begin{align}
& \Eb\left[\exp \left(\left(\tilde{\sv}_t^\TT(a)-\tilde{\sv}_{t-1}^\TT(a)\right) \Xv^\TT(a)\uv - \frac{1}{2}\uv^\TT\left(\tilde{\Vv}_t(a)-\tilde{\Vv}_{t-1}(a)\right) \uv \right) \middle| a_t=a,c_t=c\right]\\
&= \Eb\left[\exp \left(\eta_t\ev_c \Xv^\TT(a)\uv - \frac{1}{2}\uv^\TT\xv(a,c)\xv^{\TT}(a,c) \uv \right)\right]\\
&= \Eb\left[\exp \left(\eta_t\xv^{\TT}(a,c)\uv \right)\right]\cdot\exp\left(- \frac{1}{2}\uv^\TT\xv(a,c)\xv^{\TT}(a,c) \uv \right)\\
&\leq \exp \left(\frac{1}{2}\left(\xv^{\TT}(a,c)\uv\right)^2 \right)\cdot\exp\left(- \frac{1}{2}\uv^\TT\xv(a,c)\xv^{\TT}(a,c) \uv \right)=1\label{eqn:notequal0}
\end{align}
where the last inequality follows from Assumption~1.3 that $\eta_t$ is conditionally 1-subgaussian.

Combining (\ref{eqn:equal0})(\ref{eqn:notequal0}) with (\ref{eqn:martingale}), for every fixed $\uv$, we have $\Eb[M_t(\uv) \mid \Fc_{t-1}]\leq M_{t-1}(\uv)$. Thus, $\{M_{t}(\uv)\}_t$ is a super-martingale, and
\begin{align}
\Eb[{M}_t(\uv)]&\leq {M}_0(\uv)]= 1.
\end{align}
Since this holds for every $\uv$, after taking expectation with respect to $\uv$, $\bar{M}_t$ is a super-martingale as well. Thus, $\Eb[\bar{M}_t]\leq \bar{M}_0=1$.
\end{proof}

\begin{Lemma}Under Algorithm~\ref{alg:linUCB-d},
\begin{align*}
&\Pb\left[  \|\tilde{\sv}_t^\TT(a) \Xv_t^\TT(a)\|_{ \Vv^{-1}_t(a)} \geq  \sqrt{2 u+\log \frac{ \det \Vv_t(a)}{\det \Vv_0}} \right]\leq e^{-u}.
\end{align*}\label{lemma:Bt}
\end{Lemma}
\begin{proof}
Assume $h$ is the probability density function of a Gaussian distribution $\Nc(\ov, \Vv_0)$, i.e., $$h(\uv)=\frac{1}{\sqrt{(2\pi)^d \det \Vv_0}} \exp \left( -\frac{1}{2} \uv^\TT\Vv_0\uv \right).$$
Then,
\begin{align*}
\bar{M}_t& =\int_{\Rb^d} M_t(\uv ) h(\uv) \diff\uv\\
&=\frac{1}{\sqrt{(2\pi)^d \det \Vv_0^{-1}}} \int_{\Rb^d} \exp \left( \tilde{\sv}_t^\TT(a) \Xv^\TT (a)\uv - \frac{1}{2}\uv^\TT \tilde{\Vv}_t(a)\uv -\frac{1}{2} \uv^\TT \Vv_0 \uv \right) \diff \uv\\
&= \frac{1}{\sqrt{(2\pi)^d \det \Vv_0^{-1}}} \int_{\Rb^d} \exp \left( \tilde{\sv}_t^\TT(a) \Xv^\TT(a) \uv - \frac{1}{2}\uv^\TT \Vv_t(a) \uv  \right) \diff \uv\\
&=\frac{1}{\sqrt{(2\pi)^d \det\Vv_0^{-1}}} \int _{\Rb^d}\exp \left( \tilde{\sv}_t^\TT (a)\Xv^\TT(a) \Vv^{-1/2}_t(a) \Vv^{1/2}_t(a) \uv - \frac{1}{2}\|\uv^\TT \Vv_t(a)\|_{\Vv^{-1}_t(a)}  \right) \diff \uv\\
& =  \frac{1}{\sqrt{(2\pi)^d \det\Vv_0^{-1}}} \nonumber\\
&\qquad \times \int_{\Rb^d} \exp \left( \frac{1}{2}\|\tilde{\sv}_t^\TT(a) \Xv^\TT(a) \|^2_{\Vv^{-1}_t(a)} - \frac{1}{2}\|\uv^\TT \Vv_t(a)- \tilde{\sv}_t^\TT (a)\Xv^\TT(a) \|_{\Vv^{-1}_t(a)}^2 \right) \diff \uv\\
&= \sqrt{\frac{\det \Vv^{-1}_t(a)}{ \det {\Vv}^{-1}_0}}  \exp \left( \frac{1}{2}\|\tilde{\sv}_t^\TT(a) \Xv^\TT(a)\|^2_{ \Vv^{-1}_t(a)} \right) \\
&= \exp \left( \frac{1}{2}\|\tilde{\sv}_t^\TT(a) \Xv^\TT(a)\|^2_{ \Vv^{-1}_t(a)} +\frac{1}{2}\log \frac{\det\Vv_0}{ \det \Vv_t(a)}  \right).
\end{align*}

Therefore, according to Lemma~\ref{lemma:martingale}, we have
 \begin{align}
&\Pb\left[ \|\tilde{\sv}_t^\TT(a) \Xv_t^\TT(a) \|_{\Vv^{-1}_t(a)} \geq \sqrt{2u +\log \frac{ \det \Vv_t(a)}{\det{\Vv_0}} } \right]\nonumber\\
&= \Pb\left[ \frac{1}{2}\|\tilde{\sv}_t^\TT(a) \Xv_t^\TT(a) \|^2_{\Vv^{-1}_t(a)} + \frac{1}{2} \log \frac{\det {\Vv_0}}{ \det \Vv_t(a)} \geq u \right]\nonumber \\
& \leq e^{-u}\Eb[\bar{M}_t]\leq  e^{-u}.\label{eqn:tilde_s2}
\end{align}
\end{proof}

Next, we will provide a bound on $\frac{\det \Vv_0}{ \det \Vv_t(a)}$. The definition of $\Vv_0$ indicates that $\det {\Vv_0}=l^{2d}$. For $\Vv_t(a)$, we note that for any $\yv\in\Rb^{d}$, we have
\begin{align}
\yv^{\TT}\Vv_t (a)\yv & = \sum_{c=1}^n N_t(a,c)\yv^{\TT}\xv (a,c) \xv^{\TT}(a,c)\yv + l^2 \yv^{\TT}\yv\nonumber\\
&\leq t l^2 \|\yv\|_2^2+ l^2 \|\yv\|_2^2=(t+1)l^2 \|\yv\|_2^2,\label{eqn:eigen}
\end{align}
where the inequality follows from Assumption~\ref{assump:bounded}.1.

Eqn. (\ref{eqn:eigen}) indicates that the maximum eigenvalue of $\Vv_t (a)$ is upper bound by $(t+1)l^2$. Therefore, we have $\det \Vv_t(a)\leq (l^2 + t l^2)^d$, which implies that
\begin{align}\label{eqn:det}
\log \frac{ \det \Vv_t(a)}{\det \Vv_0} \leq  d\log (1+t).
\end{align}

Combining (\ref{eqn:det}) with (\ref{eqn:error2})(\ref{eqn:bar_s})(\ref{eqn:tilde_s}) and Lemma~\ref{lemma:Bt}, we have
\begin{align}
& \Pb\left[ |\hat{r}_t(a)-r(a,c_t)| \geq  \left(l s +\sqrt{2 u+d\log (1+t)} \right) \hat{\sigma}_t(a)\right]\nonumber\\
&\leq \Pb\left[  \|\tilde{\sv}_t^\TT(a) \Xv^\TT(a) \|_{\Vv^{-1}_t(a)} \hat{\sigma}_t(a) \geq  \sqrt{2 u+d\log (1+t)}  \hat{\sigma}_t(a)\right] \nonumber\\
&\leq \Pb\left[  \|\tilde{\sv}_t^\TT(a) \Xv^\TT(a) \|_{\Vv^{-1}_t(a)} \geq  \sqrt{2 u+\log \frac{ \det \Vv_t(a)}{\det \Vv_0}} \right]\leq e^{-u}.\label{eqn:ucb_bound}
\end{align}
Set $u=\log f(t)$ and $\alpha_t=l s +\sqrt{(2 +d)\log f(t)} $. When $t>2$, $f(t)>1+t$, therefore, (\ref{eqn:ucb_bound}) implies that
\begin{align}
\Pb\left[ |\hat{r}_t(a)-r(a,c_t)| \geq  \alpha_t \hat{\sigma}_t(a)\right]\leq \frac{1}{f(t)}.
\end{align}
Thus,
\begin{align}
\Eb[|\Bc_T|]&\leq 2+K\sum_{t=3}^\infty  \frac{1}{f(t)} \leq 2+2.5K,\quad \Eb[R(\Bc_T)]\leq M \Eb[|\Bc_T|]\leq M(2+2.5K).\label{eqn:Bt}
\end{align}

\if{0}
\begin{theorem}\label{thm:Bt}
Under Algorithm~\ref{alg:linUCB-d}, we have
\begin{align}
\Eb[\Bc_T]&\leq 2+K\sum_{t=3}^\infty  \frac{1}{f(t)} \leq 2+2.5K,\quad \Eb[R(\Bc_T)]\leq M \Eb[\Bc_T]\leq M(2+2.5K). \label{eqn:Bt}
\end{align}
\end{theorem}
The detailed proof of Theorem~\ref{thm:Bt} can be found in Appendix~\ref{appx:thm:Bt}.
\fi

\if{0}
We have the following observation.
\begin{Proposition}
$\|\tilde{\sv}_t^\TT(i) \Xv^\TT(i)\Vv_t^{-1/2}(i)\|_2 =\|\hat{\thetav}_t(i)-\thetav_i\|_{\Vv_t^{-1}(i)}$, where $\|\xv\|_\Vv:=\sqrt{\xv^T\Vv\xv}$.
\end{Proposition}
\begin{proof}
According to Proposition~\ref{prop:equivalence}, we have
\begin{align}
\hat{\thetav}_t(i)&= \Vv_t^{-1}(i) \Xv(i)\sv_t(i)
\end{align}
Thus,
\begin{align}
\hat{\thetav}_t(i)-\thetav_i&= \Vv_t^{-1}(i) \left[\Xv(i)\sv_t(i)-\Vv_t(i) \thetav_i\right]\\
&= \Vv_t^{-1}(i) \left[\Xv(i)\sv_t(i)-\Xv(i) \Nv_t(i) \Xv^{\TT}(i) \thetav_i\right]
\end{align}
\end{proof}
\fi


\subsection{Bound the Regret over $\Cc_T$}\label{appx:Ct}
Recall $B_k:=|\Bc_T\cap F_k|$, i.e., the number of bad estimates in frame $k$. Then, according to Markov's inequality, we have 
\begin{align}\label{eqn:BkC}
\Pb \left[ B_k \geq \frac{2^{k-1}}{4n} \right] & \leq \frac{\Eb[B_k] 4 n}{2^{k-1}}.
\end{align}
The definitions of $B_k$ and $\Bc_T$ also imply that $\sum_{k=1}^{\lceil\log_2 T\rceil} \Eb[B_k] = \Eb[\Bc_T].$
Therefore, 
\begin{align}
 \Eb[R(\Cc_T)]& \leq M\E [|\Cc_T| ]=M\sum_{k=1}^{\lfloor\log_2 T\rfloor}  |F_{k+1}|\cdot \Pb \left[ B_k \geq \frac{2^{k-1}}{4n} \right]\label{eqn:CT1}  \\
&	\leq M\sum_{k=1}^{\lfloor\log_2 T\rfloor}  2^{k}  \frac{\Eb[B_k] 4 n}{2^{k-1}} \leq 8  n M \Eb[|\Bc_T|]  \leq  8 n  M (2+2.5K),\label{eqn:CT31}
\end{align}
where (\ref{eqn:CT1}) follows from the definition of $\Cc_T$, and (\ref{eqn:CT31}) is due to (\ref{eqn:BkC}) and (\ref{eqn:Bt}).

\subsection{Bound the Regret over $\Dc_T$}\label{appx:Dt}
Let $\bar{N}_t(a,c)$ be the total number of time slots before $t$ when arm $a$ is pulled under context $c$, and all estimates are good, i.e.,
$\bar{N}_t(a,c)=|\{\tau\mid a_\tau=a, c_\tau=c, \tau\notin \Bc_t, 1\leq \tau<t\}|.$
We have the following observations.
\begin{Lemma}
\label{lemma_nu}
For any $a$, $c \notin \Cc_a$, $\bar{N}_t(a,c)\leq \frac{{4} \alpha_t^2}{\Delta^2}$ for all $t$.
\end{Lemma}

\begin{proof}
We first consider a time slot $t\notin\Bc_T$ at which a sub-optimal action $a_t$ is taken under $c_t$. Then,
according to the LinUCB-d Algorithm, we must have
\begin{align}\label{eqn:ucb1}
 \hat{r}_t(a_t) + \alpha_t\hat{\sigma}_t(a_t)\geq  \hat{r}_t(a^*_t) + \alpha_t\hat{\sigma}_t(a^*_t).
\end{align}

Besides, $t\notin\Bc_T$ implies that
\begin{align}
|\hat{r}_t(a_t)-r(a_t,c_t)|&\leq \alpha_t\hat{\sigma}_t(a_t),\quad
|\hat{r}_t(a^*_t)-r(a^*_t,c_t)|\leq \alpha_t\hat{\sigma}_t(a^*_t).\label{eqn:ucb3}
\end{align}

Putting (\ref{eqn:ucb1})(\ref{eqn:ucb3}) together, we have
\begin{align}
& r(a_t,c_t)+2 \alpha_t\hat{\sigma}_t(a_t)\geq  \hat{r}_t(a_t) + \alpha_t\hat{\sigma}_t(a_t) \geq  \hat{r}_t(a^*_t) + \alpha_t\hat{\sigma}_t(a^*_t)\geq r(a^*_t,c_t).
\end{align}

Therefore,
\begin{align}
&\Delta \leq r(a^*_t,c_t) - r(a_t,c_t) \leq 2 \alpha_t\hat{\sigma}_t(a_t) = 2 \alpha_t \sqrt{\betav_t^\TT(a_t) \Nv_t^{-1}(a_t) \betav_t(a_t)}. \label{eqn:ucb_gap}
\end{align}

Denote $\tilde{\betav}\in\Rb^{n_t+d}$ as a unit vector whose $c_t$-th entry takes value $1$. Then, when $N_t(a_t,c_t)\neq 0$, $\tilde{\betav}$ satisfies the constraints in (\ref{eqn:beta}). According to Proposition~\ref{prop:beta}, we must have
\begin{align}
&\betav_t^\TT(a_t) \Nv_t^{-1}(a_t) \betav_t(a_t)\leq \tilde{\betav}^\TT(a_t) \Nv_t^{-1}(a_t) \tilde{\betav}(a_t)= \frac{1}{N_t(a_t,c_t)} \leq \frac{1}{\bar{N}_t(a_t,c_t)}.\label{eqn:timebound}
\end{align}
Combining (\ref{eqn:ucb_gap}) and (\ref{eqn:timebound}), we have
\begin{align}\label{eqn:barN}
\bar{N}_t(a_t,c_t)&\leq \frac{1}{\betav_t^\TT(a_t) \Nv_t^{-1}(a_t) \betav_t(a_t)}\leq \frac{4 \alpha^2_t }{\Delta ^2}.
\end{align}
When $N_t(a_t,c_t)=0$, we must have $\bar{N}_t(a_t,c_t)=0$, thus (\ref{eqn:barN}) is satisfied as well.

Hence, Lemma \ref{lemma_nu} holds for all time slots $t\notin \Bc_T$. Since $\bar{N}_t(a,c)$ is a step function for any fixed $(a,c)$ pair and $\alpha_t$ monotonically increases in $t$, Lemma \ref{lemma_nu} hold for all $t$ as well.
\end{proof}

Lemma~\ref{lemma_nu} indicates that the total number of times that $a$ is pulled as a sub-optimal arm up to $t$ is bounded by $O(\log f(t))$. Based on this result, we will then show that the total number of times that $a$ is pulled as an optimal arm grows linearly in $t$, as described in Lemma~\ref{lemma_tau}. Next, we utilize Lemma~\ref{lemma_tau} to show the diminishing estimation uncertainty in Lemma~\ref{lemma_beta}, which eventually leads to the finite regret bound over $\Dc_T$ in Theorem~\ref{thm:Dt}. 


\begin{Lemma}
\label{lemma_tau}
For any $a$, $c \in \Cc_a$ and any time slot $t\in \Dc_T$, we must have $ N_t(a,c)\geq \frac{t}{16n} - \frac{8K\alpha_t^2}{\Delta^2}$.
\end{Lemma}
\begin{proof}
Assume $t$ lies in the $(k+1)$th time frame. Then, based on the definition of $N_t(a,c)$, we must have
\begin{align}
  N_t(a,c) & \geq  N_{F_k}(a,c) = N_{F_k}(c) - \sum_{b:b \neq a} N_{F_k}(b,c)  \geq \frac{2^{k-1}}{2n} - \Big[B_k + \sum_{b:b \neq a} \bar{N}_{2^k}(b,c)\Big] \label{eqn:goodtime1}\\
 & \geq \frac{2^{k-1}}{2n} - \frac{2^{k-1}}{4 n} - K \frac{4 \alpha_t^2}{\Delta^2} \geq \frac{t}{16n} -  \frac{4 K \alpha_t^2}{\Delta^2},\label{eqn:goodtime2}
\end{align}
where (\ref{eqn:goodtime1}) follows from the assumption that $t\notin\Ac_T$, and (\ref{eqn:goodtime2}) follows from the assumption that $t\notin\Cc_T$ and Lemma~\ref{lemma_nu}.
\end{proof}

Before we proceed, we introduce the following lemma.

\begin{Lemma}\label{lemma:span}
Let $\{\zv_1,\zv_2,\ldots,\zv_d\}$ be a basis for $\Rb^d$, and $\Zv:=[\zv_1,\zv_2,\ldots,\zv_d]$. Then, for any $\xv\in\Rb^d$, $\|\xv\|_2\leq l$, we can express it as
$\xv=\Zv \betav$, where $\betav\in \Rb^d$, $\|\betav\|_1\leq\frac{l\sqrt{d}}{\sqrt{\lambda_{\min}( \Zv^\TT\Zv)}}$.
\end{Lemma}

\begin{proof}
Since 
\begin{align}\label{eqn:cont_eigen_0}
l^2\geq \|\xv\|_2^2=\betav^\TT \Zv^\TT\Zv \betav\geq \lambda_{\min}( \Zv^\TT\Zv)\betav^\TT\betav\geq \frac{\lambda_{\min}( \Zv^\TT\Zv) \|\betav\|_1^2 }{d},
\end{align}
we have $\|\betav\|_1 \leq \frac{l\sqrt{d}}{\sqrt{\lambda_{\min}( \Zv^\TT\Zv)}}$.
\end{proof}

\begin{Lemma}
\label{lemma_beta}
For any arm $a\in[K]$, any time slot $t\in\Dc_T$, we must have $\betav_t^\TT(a) \Nv_t^{-1}(a) \betav_t(a) \leq \frac{\delta^2}{\frac{t}{16 n} - \frac{4 K \alpha_t^2}{\Delta^2}}$, where $\delta:=l\sqrt{d/\lambda_0}$.

\end{Lemma}

\begin{proof}
For any $a\in[K], c\in\Cc$, let $\bar{\betav}(a,c)$ be the solution to the following equation 
\begin{eqnarray}
\xv(a,c)=\Xv_t(a)\bar{\betav}, \quad\bar{\betav}[c] = 0, \mbox{ for }c \notin \bar{\Cc}_a. \label{opt:Ca} 
\end{eqnarray}
Note that we use $\bar{\betav}[c]$ to denote the entry associated with context $c$ in $\bar{\betav}$. 

Consider a time slot $t\in\Dc_T$. Based on the definitions of the error events in Section~\ref{sec:uniform}, we note that all contexts in $\bar{\Cc}_a$ must have appeared before time slot $t$. Thus, $\Xv_t(a)$ contains all columns in $\bar{\Zv}_a$.   
Therefore, $\bar{\betav}(a,c)$ is simply the coefficient vector if we express $\xv(a,c)$ as a linear combination of the feature vectors in $\bar{\Zv}_a$. The diversity assumption in Assumption~\ref{assump:bounded}.5 guarantees that there exists a unique solution $\bar{\betav}(a,c)$ for each $(a,c)$ pair. Besides, Lemma~\ref{lemma:span} implies that $\|\bar{\betav}(a,c)\|_{1}\leq \frac{l\sqrt{d}}{\sqrt{\lambda_{\min}(\bar{\Phi}_a^\TT \bar{\Phi}_a)}}$.

Then, according to Proposition~\ref{prop:beta}, Lemma~\ref{lemma_tau} and Lemma~\ref{lemma:span}, we must have
\begin{align}
&\betav_t^\TT(a) \Nv_t^{-1}(a) \betav_t(a) 	\leq \bar{\betav}^\TT(a,c_t) \Nv_t^{-1}(a) \bar{\betav}(a,c_t) \leq \frac{ \|\bar{\betav}(a,c_t)\|_1^2 }{\frac{t}{16 n} -  \frac{4 K \alpha_t^2}{\Delta^2}} 
  \leq \frac{\delta^2}{\frac{t}{16 n} -  \frac{4 K \alpha_t^2}{\Delta^2}},\label{eqn:variance}
\end{align}
where the first inequality in (\ref{eqn:variance}) follows from Proposition~\ref{prop:beta}, the second inequality follows from Lemma~\ref{lemma_tau}, and the last inequality follows from Lemma~\ref{lemma:span}.
\end{proof}

We then have the following bound on $\Eb[R(\Dc_T)]$.
\begin{theorem}\label{thm:Dt} Let  
$$t_1 = \max \left\{ \frac{384(2+d)n(\delta^2+K)}{\Delta^2}, 10 \right\} ,\quad t_2  = \max\left\{t_1 \log t_1 ,  \exp \left(\frac{12l^2s^2}{2+d} \right)\right\}.$$
Then, under Algorithm~\ref{alg:linUCB-d}, 
\begin{align*}
\Eb[R(\Dc_T)]\leq t_2 M = O\left(\frac{dn(\delta^2+K)}{\Delta^2}\log \frac{dn(\delta^2+K)}{\Delta^2}\right). 
\end{align*}
\end{theorem}

\begin{proof}
For any $t\geq t_2$, we have
\begin{align*}
t\geq \exp \left(\frac{12l^2s^2}{2+d} \right)\geq \exp \left(\frac{(\sqrt{2}+\sqrt{3})^2l^2s^2}{2+d} \right),
\end{align*}
which implies that
\begin{align}\label{eqn:ls}
ls\leq \sqrt{2+d}(\sqrt{3}-\sqrt{2})\sqrt{\log t}.
\end{align}
Meanwhile, since $\log f(t)\leq 2\log t$, combining with (\ref{eqn:ls}), we have
\begin{align}
\alpha_t:=ls+\sqrt{(2+d)\log f(t)}\leq \sqrt{3(2+d)\log t} .\label{eqn:alpha_t}
\end{align}

\tcr{Since $\frac{t_2}{\log t_2}  \geq \frac{t_1 \log t_1}{\log (t_1 \log t_1)} \geq \frac{t_1}{2}$,} for any $t\geq t_2$, we have
\begin{align}
 \frac{t}{\log t}   \geq \frac{t_1}{2 }=\frac{192(2+d)n(\delta^2+K) }{\Delta^2},\label{eqn:T0}
 \end{align}
i.e.,
\begin{align*}
t&>\frac{192(2+d)n(\delta^2+K) }{\Delta^2}\log t\geq  \frac{64n(\delta^2+K) }{\Delta^2}\alpha_t^2,
\end{align*}
where the last inequality is due to (\ref{eqn:alpha_t}).

Thus,
\begin{align}
\frac{t}{16n}&\geq \frac{4\alpha_t^2\delta^2}{\Delta^2}+\frac{4K\alpha_t^2}{\Delta^2}.
\end{align}
Rearranging the terms, we have
\begin{align}
\Delta^2 &>\frac{4\alpha_t^2\delta^2}{\frac{t}{16n}-\frac{4K\alpha_t^2}{\Delta^2}}\geq (2\alpha_t\hat{\sigma}_t(a))^2,\quad\forall a\in[K],\label{eqn:alpha}
\end{align}
where the last inequality follows from Lemma~\ref{lemma_beta}.

Since $\Delta\geq 2\alpha_t\hat{\sigma}_t(a)$ for any $t\geq t_2$, arm $a$ will not be pulled as a suboptimal arm at any time $t\notin\Bc_T$, according to Eqn.~(\ref{eqn:ucb_gap}). Therefore, $\Dc_T$ can only include time indices before $t_2$. The expected regret over $\Dc_T$ can thus be bounded by $Mt_2$.
\end{proof}

\subsection{Put Everything Together}
After obtaining bounds on the expected regret over $\Ac_T$, $\Bc_T$, $\Cc_T$ and $\Dc_T$, we are ready to prove our main result in Theorem~\ref{thm:main}. We have
\begin{align}
\Eb[R_T]&\leq \Eb[R(\Ac_T)]+\Eb[R(\Bc_T)]+\Eb[R(\Cc_T)]+\Eb[R(\Dc_T)]\nonumber\\
&\leq 8MKdn^2 +(8 n+1)  M (2+2.5K)+t_2 M \\
&=O\left(Kdn^2+\frac{dn(K+\delta^2)}{\Delta^2}\log \frac{dn(K+\delta^2)}{\Delta^2}\right). \label{eqn:thm2}
\end{align}

We point out that the $O\left(\exp \left(\frac{12l^2s^2}{2+d} \right)\right)$ term from $t_2$ is dropped in (\ref{eqn:thm2}), since it mainly depends on the bounds on $\|\thetav(a)\|_2$ and $\|\xv(a,c)\|_2$, and does not scale with the system dimensions $d$ or $K$.

\if{0}
Let $t^*(a,c)$ be the last time slot in $\Dc_T$ when arm $a$ is pulled under context $c$. Then,
The expected total regret over $\Dc_T$ can be bounded by
\begin{align}
&\Eb[R(\Dc_T)]\nonumber\\
&=\sum_{a\in[K]} \sum_{c\notin \Cc_a} \Eb[( \lv\{t^*(a,c)\leq T_0\} + \lv\{t^*(a,c)>T_0\} )\bar{N}_{t^*}(a,c) [ r(a^*, c) - r(a,c)]]\\
 &\leq T_0 M + \sum_{a\in[K]} \sum_{c\notin \Cc_a} \Eb\left[ \lv\{t^*(a,c)>T_0\} \bar{N}_{t^*(a,c)}(a,c) ( r(a^*, c) - r(a,c) ) \right]  \\
 & \leq  T_0 M + \sum_{a\in[K]} \sum_{c\notin \Cc_a}\Eb\left[ \lv\{t^*(a,c)>T_0\} \frac{8\alpha_{t^*(a,c)}^3}{\Delta^2} \sqrt{\frac{\delta^2 }{\frac{t^*(a,c)}{16n}-\frac{4K\alpha_{t^*(a,c)}^2}{\Delta^2}}} \right]  \label{eqn:regD1} \\
 &\leq T_0 M + \sum_{a\in[K]} \sum_{c\notin \Cc_a} \Eb\left[ \lv\{t^*(a,c)>T_0\} \frac{32\delta\sqrt{n}}{\Delta^2} \sqrt{ \frac{(2(3+d))^3\log^3 t^*(a,c) }{t^*(a,c)-\frac{128(3+d)nK \log t^*(a,c)}{\Delta^2}}}\right]\label{eqn:regD2}\\
 &\leq T_0 M + \frac{32\delta n^{3/2} K (2(3+d))^{3/2}} {\Delta^2} \sqrt{ \frac{\log^3 T_0 }{T_0-\frac{128(3+d)nK \log T_0}{\Delta^2}}}\label{eqn:regD3}\\
 &\leq T_0 M + \frac{32\delta n^{3/2} K (2(3+d))^{3/2}}{\Delta^2}\frac{ \log T_0 }{ \sqrt{nK/\Delta^2}}\label{eqn:regD4}\\
 &\leq T_0 M +\frac{64(3+d)n\delta}{\Delta} \sqrt{2(3+d)K} \log T_0
\end{align}
where $a^*$ represents the optimal arm for a given context $c$, (\ref{eqn:regD1}) is based on Lemma~\ref{lemma_nu}, eqn. (\ref{eqn:ucb_gap}) and Lemma~\ref{lemma_beta}; (\ref{eqn:regD2}) follows from the upper bound in (\ref{eqn:alpha}); (\ref{eqn:regD3}) is due to the fact that the radicand in (\ref{eqn:regD2}) is monotonically decreasing in $t^*(i,c)$ when it is greater then $T_0$; (\ref{eqn:regD4}) is based on the upper bound of $\frac{T_0}{\log T_0} $ in (\ref{eqn:T0}).
\fi

\section{Proof of Theorem~\ref{thm:continuous}}\label{appx:thm:continuous}


Before we proceed, we will first introduce the following lemma, which will play a critical role in the analysis afterwards.

\begin{Lemma}\label{lemma:span3}
Let $\{\zv_1,\zv_2,\ldots,\zv_d\}$ be a basis for $\Rb^d$, and $\Zv:=[\zv_1,\zv_2,\ldots,\zv_d]$. 
Let $B(\zv_i,r)$ be an $\ell_2$ ball centered at $\zv_i$ with radius $r< \sqrt{\lambda_{\min}(\Zv^\TT\Zv)/d} $, i.e., 
$B(\zv_i,r):=\{\xv\in \Rb^d \mid \|\xv-\zv_i\|_2\leq r\}$.
Let $\hat{\zv}_i$ be any vector lying in $B(\zv_i,r)$ and $\hat{\Zv}:=[\hat{\zv}_1,\hat{\zv}_2,\ldots,\hat{\zv}_d]$. Then, $\lambda_{\min}(\hat{\Zv}^\TT\hat{\Zv})\geq (\sqrt{\lambda_{\min}(\Zv^\TT\Zv)}-\sqrt{d}r)^2$.
\end{Lemma}

\begin{proof}
Denote $\rv_i:=\hat{\zv}_i-\zv_i$. Then, based on the definition of $\hat{\zv}_i$, we have $\|
\rv_i\|_2\leq r$. Let $\Rv=[\rv_1,\rv_2,\ldots,\rv_d]$, and $\Rv(j)$ be its $j$th row. Then, for any $\betav\in\Rb^d$,
\begin{align}
\|\hat{\Zv}\betav\|_2&= \left\| \Zv \betav+\Rv\betav \right\|_2\geq \left\|\Zv \betav\right\|_2- \left\|\Rv\betav \right\|_2\geq \sqrt{\lambda_{\min}(\Zv^\TT\Zv)} \|\betav\|_2 -  \sqrt{\sum_{j} \left|\Rv(j)\betav \right|^2} \label{eqn:cont_eigen}\\
&\geq \sqrt{\lambda_{\min}(\Zv^\TT\Zv)} \|\betav\|_2 -  \sqrt{\sum_{j} \left\|\Rv(j)\|_2^2 \|\betav \right\|_2^2} \label{eqn:cont_cauchy}\\
& = \sqrt{\lambda_{\min}(\Zv^\TT\Zv)} \|\betav\|_2 - \|\betav\|_2  \sqrt{\sum_{i=1}^d \rv_i^\TT\rv_i }\label{eqn:cont_rearrange}\\
&\geq \left(\sqrt{\lambda_{\min}(\Zv^\TT\Zv)}-\sqrt{d}r\right) \|\betav\|_2,\label{eqn:cont_bound}
\end{align}
where (\ref{eqn:cont_eigen}) follows from (\ref{eqn:cont_eigen_0}), and (\ref{eqn:cont_cauchy}) follows from the Cauchy-Schwartz inequality; Rearranging the terms involved in the summation, we obtain (\ref{eqn:cont_rearrange}), which can be further bounded by (\ref{eqn:cont_bound}) due to the definition of $B(\zv_i,r)$.

Thus, the eigenvalues of $\hat{\Zv}^\TT\hat{\Zv}$ are lower bounded by $(\sqrt{\lambda_{\min}(\Zv^\TT\Zv)}-\sqrt{d}r)^2>0$.
\end{proof}

\textbf{Remark:} Lemma~\ref{lemma:span3} implies that $\{\hat{\zv}_i\}$ are linearly independent, thus forming a valid basis for $\Rb^d$. 

\if{0}
Next, we will leverage Lemma~\ref{lemma:span} to obtain a valid basis for $\Xc_a$, $a\in[K]$. 

For each arm $a$, we first fix a basis $\Zc_a:=\{\zv_{a,1},\zv_{a,2},\ldots,\zv_{a,d}\}$ for $\Xc_a$. Under the diversity condition in Assumption~\ref{assump:diversity}, such a basis always exists. We then divide $\Xc_a$ into $d$ group $\Xc_{a,1}$, $\Xc_{a,2}$, $\ldots$, $\Xc_{a,d}$ based their closeness to the basis vectors, and break the tie by assigning $\xv$ to the group with larger index, i.e.,
\begin{align}
\Xc_{a,i}&=\left\{ \xv\in \Xc_a \middle|  \frac{\xv^\TT \zv_{a,i}}{\|\zv_{a,i}\|_2} <  \frac{\xv^\TT \zv_{a,j}}{\|\zv_{a,j}\|_2} \mbox{ for } j<i,  \frac{\xv^\TT \zv_{a,i}}{\|\zv_{a,i}\|_2} \leq  \frac{\xv^\TT \zv_{a,j}}{\|\zv_{a,j}\|_2} \mbox{ for } j>i \right\}.
\end{align}
Let $\Zv_a$ be the corresponding matrix consisting of $\{\zv_{a,i}\}$, and $\lambda_0$ be the smallest number among the eigenvalues of matrices $\Zv_a$s. 

Let $r$ be a number lying in $(0,\sqrt{\lambda_0/d})$ and $\hat{\Xc}_{a,i}:=\Xc_{a,i}\cap B(\zv_{a,i},r)$. Then, based on the results in Lemma~\ref{lemma:span}, if we randomly pick a vector $\hat{\zv}_{a,i}\in \hat{\Xc}_{a,i} $ for $i=1,2,\ldots,d$, $\left\{\hat{\zv}_{i,a}\right\}_i$ still form a valid basis for $\Xc_a$.

 Lift $\Xc_{a}^{(i)}$ and $\bar{\Xc}_{a}^{(i)}$ to $\Cc_a$ as follows
\begin{align}
{\Cc}_{a}^{(i)}&:= \{c\in\Cc_a\mid \xv(a,c)\in\Xc_{a}^{(i)}\}\\
\hat{\Cc}_{a}^{(i)}&:= \{c\in\Cc_a\mid \xv(a,c)\in\hat{\Xc}_{a}^{(i)} \}
\end{align}

Let $p:=\min_{a}^{(i)} \Pb[\hat{\Cc}_{a}^{(i)}]$, and $\rho:= \frac{\Pb[{\Cc}_{a}^{(i)}]}{\Pb[\hat{\Cc}_{a}^{(i)}]}$. Modify the definition of $\Ac_T$ and $\Cc_T$ as follows
\begin{align}
\Ac_T&:=\cup_k \left\{F_{k+1}\mid \exists i,a, \mbox{ s.t. } N_{F_k}(\hat{\Cc}_{a}^{(i)})\leq \left(\frac{p}{2}\right)2^{k-1} \right\},\\
\Cc_T&:=\cup_k  \left\{ |\Bc_T\cap F_k|\geq \left(\frac{p}{4}\right)2^{k-1} \right\},
\end{align}
and keep the definitions of $\Bc_T$ and $\Dc_T$ the same.

Analogous to the finite case, $\hat{\Cc}_{a}^{(1)}$, $\Cc_{a}^{(1)}-\hat{\Cc}_{a}^{(1)}$, $\ldots$, $\hat{\Cc}_{a}^{(d)}$, ${\Cc}_{a}^{(d)}-\hat{\Cc}_{a}^{(d)}$ serve as $n=2Kd$ meta-contexts.

The idea is to show that the total number of times that arm $a$ is pulled as the optimal arm under contexts in $\hat{\Cc}_{a}^{(i)}$ scales linearly in $t$ (Lemma~\ref{lemma:cont_fraction}). In order to achieve this, we will show that over $\Dc_T$, the number of times that a suboptimal arm $b\neq a$ is selected under contexts in $\bar{\Cc}_{a}^{(i)}$ grows sublinearly in $t$ (Lemma~\ref{lemma:cont_freq}). The extra factor $d\log \frac{d+t}{d}$ is due to the randomness in $\hat{\Cc}_{a}^{(i)}$, which is no longer a single context. 

The average of previously observed vectors in $\hat{\Cc}_{a,i}$ serves the role of the basis in the finite case. Recall that $r$ is chosen such that $\{\hat{\zv}_{a,i}\}$ is still a valid basis, so we still have Lemma~\ref{lemma:cont_sigma} hold.
\fi

\subsection{Bound the Regret over $\Ac_T$}
First, based on Hoeffding's inequality, we have
\begin{align}
&\Pb \left[ N_{F_k}(\bar{\Cc}_{a}^{(i)}) \leq \frac{p}{2 }\cdot 2^{k-1} \right] \leq \exp\left( -p^2 2^{k-2} \right) .
\end{align}
Recall that $M$ is the maximum per-step regret. Thus, by extending the proof in Appendix~\ref{appx:At}, we have 
\begin{align}
& \Eb[R(\Ac_T)]  \leq M (Kd) \sum_{t=2}^{\infty} \exp \left(-\frac{p^2}{8}t\right)\leq \frac{8MKd}{p^2}.
\end{align}

\subsection{Bound the Regret over $\Cc_T$}
According to Markov's inequality, we have
\begin{align}\label{eqn:Bk}
\Pb \left[ B_k \geq \frac{p\cdot 2^{k-1}}{4} \right] & \leq \frac{\Eb[B_k] \cdot 4 }{p \cdot 2^{k-1}}.
\end{align}
Therefore, by following similar steps in Appendix~\ref{appx:Ct}, we have 
\begin{align}
 \Eb[R(\Cc_T)] 	& \leq \frac{8  M}{p} \Eb[\Bc_T]  \leq  \frac{8 M (2+2.5K)}{p}.\label{eqn:CT3}
\end{align}

\subsection{Bound the Regret over $\Dc_T$}
Before we proceed, we first state an adapted version of the celebrated elliptical potential lemma below, which will play a key role to analysis afterwards.

\begin{Lemma}[Elliptical Potential~\citep{LS19bandit-book}]\label{lemma:elliptical}
 Let $\Vv_0$ be positive definite and $\Vv_t=\Vv_{t-1}+\xv_t\xv_t^\TT$, where $\xv_1, \ldots , \xv_n\in\Rb^d$ is a sequence of vectors with $\|\xv_t\|_2 \leq l<\infty$ for all $t$. Then,
\begin{align*}
\sum_{t=1}^n \left(1\wedge \|\xv_t\|^2_{\Vv_{t-1}^{-1}}  \right)\leq 2\log \left(\frac{\det \Vv_n}{\det \Vv_0}\right)\leq 2d \log \left(\frac{\mbox{trace}\Vv_0+nl^2}{d\det^{1/d} \Vv_0}\right),
\end{align*}
where $x\wedge y =\min\{x,y\}$.
\end{Lemma}

Let $\Tc_t(a,\Cc_{b}^{(i)})$ be the time slots before $t$ when arm $a$ is pulled under a context lying in $\Cc_{b}^{(i)}$, and at the same time, all estimates are good, i.e., 
\begin{align}
\Tc_t(a,\Cc_{b}^{(i)})&:=\{\tau \mid a_\tau=a, c_\tau\in \Cc_{b}^{(i)}, \tau\notin \Bc_t, 1\leq \tau< t \},
\end{align}
and denote $\bar{N}_t(a,\Cc_{b}^{(i)}):=|\Tc(a,\Cc_{b}^{(i)})|.$

We have the following lemma analogue to Lemma~\ref{lemma_nu}.
\begin{Lemma}\label{lemma:cont_freq}
For any $a$, $b\neq a$, 
$\bar{N}_t(a,\Cc_{b}^{(i)})\leq \frac{8\alpha_t^2}{\Delta^2} d\log  \left(\frac{d+t}{d}\right)$ for all $t$.
\end{Lemma}
\begin{proof}
First, following steps similar to the proof of Lemma~\ref{lemma_nu}, for any $\tau\in \Tc_t(a,\Cc_{b}^{(i)}) $, we have
\begin{align}\label{bound1}
\Delta & \leq  r(a_\tau^*,c_\tau)-r(a,c_\tau) \leq 2\alpha_\tau \hat{\sigma}_\tau (a). 
\end{align}

Next, we consider the solution to the following optimization problem, denoted as $\tilde{\betav}_t(a)$:
\begin{eqnarray}
\min_{\betav \in \mathbb{R}^{n_t+d}} & & \betav^\TT  \Nv^{-1}_t(a) \betav,\quad \mbox{s.t. } \quad \xv(a,{c_t}) =\Xv_t(a) \betav, \quad \betav[c]=0 \mbox{ if } c\notin \Cc_{b}^{(i)}\cup \Cc_0.\label{eqn:beta_cont} 
\end{eqnarray}

Compared with the optimization problem in (\ref{eqn:beta}), we have one additional constraint, i.e., we only restrict to the contexts in $ \Cc_{b}^{(i)}$ and $\Cc_0$. The inclusion of the dummy contexts $\Cc_0$ ensures the existence of at least one feasible solution to (\ref{eqn:beta_cont}). Due to the additional constraint, the corresponding minimum value of the objective function must increase, i.e.,
\begin{align}\label{bound2}
\tilde{\sigma}_t(a)&:=\sqrt{\tilde{\betav}_t^\TT(a)  \Nv^{-1}_t(a)\tilde{\betav}_t(a)}\geq \hat{\sigma}_t(a), \quad \forall t.
\end{align} 

Note that
\begin{align}\label{bound3}
r(a_\tau^*,c_\tau)-r(a,c_\tau)\leq 2ls\leq 2\alpha_\tau,
\end{align}
where the first inequality in (\ref{bound3}) follows from Assumption 1.1 and the second inequality follows from the definition of $\alpha_t$.

Combining (\ref{bound2})(\ref{bound3}) with (\ref{bound1}), we have
\begin{align}
\Delta \leq 2 \alpha_\tau \left( 1\wedge \tilde{\sigma}_\tau(a )\right), \quad \tau\in \Tc_t(a,\Cc_{b,i}) .
\end{align}

Summing over all $\tau\in \Tc_t(a,\Cc_{b}^{(i)}) $, we have
\begin{align}
 \bar{N}_t(a,\Cc_{b}^{(i)})\Delta &\leq \sum_{\tau\in \Tc_t(a,\Cc_{b}^{(i)}) } 2\alpha_\tau\left( 1\wedge \tilde{\sigma}_\tau (a)\right) \nonumber\\
 &\leq 2\alpha_t \sqrt{\bar{N}_t(a,\Cc_{b}^{(i)}) \Bigg(\sum_{\tau\in \Tc_t(a,\Cc_{b}^{(i)}) } \left(1\wedge \tilde{\sigma}_\tau (a)\right)^2\Bigg) }, \label{eqn:cont_barN}
\end{align}
where (\ref{eqn:cont_barN}) follows from the monotonicity of $\alpha_t$ and the Cauchy-Schewartz inequality.

Consider the sequence of feature vectors $\{\xv(a,c_\tau)\}_{\tau\in \Tc_t(a,\Cc_{b}^{(i)})}$. Label the times indices in $\Tc_t(a,\Cc_{b}^{(i)})$ as $\tau_1$, $\tau_2$, $\ldots$. Let $\tilde{\Vv}_0=l^2 \mathbf{I}$,  $\tilde{\Vv}_{\tau_i}=\tilde{\Vv}_{\tau_{i-1}} +\xv(a,c_{\tau_i}) \xv(a,c_{\tau_i})^\TT$. Then, similar to (\ref{eqn:sigma}), we have $\tilde{\sigma}_{\tau_i} (a)=\|\xv(a,c_{\tau_i})\|_{\tilde{\Vv}^{-1}_{\tau_{i-1}}}$. Following Lemma~\ref{lemma:elliptical}, we have
\begin{align}\label{eqn:bound4}
\sum_{\tau\in \Tc_t(a,\Cc_{b}^{(i)}) } \left(1\wedge \left(\tilde{\sigma}_\tau (a)\right)^2\right) \leq 2d\log \left(\frac{dl^2+\bar{N}_t(a,\Cc_{b}^{(i)})l^2}{dl^2}\right)\leq 2d\log \left(\frac{d+t}{d}\right).
\end{align}
Plugging (\ref{eqn:bound4}) into (\ref{eqn:cont_barN}) and rearranging the terms, we have
$\bar{N}_t(a,\Cc_{b}^{(i)})\leq \frac{8\alpha_t^2}{\Delta^2} d\log \left(\frac{d+t}{d}\right)$ for all $t$.
\end{proof}

\begin{Lemma}\label{lemma:cont_fraction}
For any $a\in [K]$, any time slot $t\in\Dc_T$,
$N_t(a,\bar{\Cc}_{a}^{(i)})\geq \frac{tp}{16}-\frac{8K\alpha_t^2}{\Delta^2}d\log  \left(\frac{d+t}{d}\right)$.
\end{Lemma}
\begin{proof}
Assume $t$ lies in the $(k+1)$th time frame. Then, based on the definition of $N_t(a,\bar{\Cc}_{a}^{(i)}) $, we have
\begin{align}
N_t(a,\bar{\Cc}_{a}^{(i)}) &\geq N_{F_k}(a,\bar{\Cc}_{a}^{(i)})\geq N_{F_k}(\bar{\Cc}_{a}^{(i)})-\sum_{b: a\neq b} N_{F_k}(b,\bar{\Cc}_{a}^{(i)})\label{eqn:cont_D1}\\
&\geq \frac{2^{k-1}}{2}p-\left[B_k+\sum_{b: a\neq b} \bar{N}_{2^k}(b,{\Cc}_{a}^{(i)})\right]\label{eqn:cont_D2}\\
& \geq \frac{2^{k-1}}{2}p-\frac{2^{k-1}}{4}p-K\frac{8\alpha_t^2}{\Delta^2} d\log \left(\frac{d+t}{d}\right)\label{eqn:cont_D23}\\
&= \frac{tp}{16}-\frac{8K\alpha_t^2}{\Delta^2}d\log  \left(\frac{d+t}{d}\right),
\end{align}
where (\ref{eqn:cont_D1}) follows from the assumption that $t\notin\Ac_T$, (\ref{eqn:cont_D2}) follows from the fact that $\bar{\Cc}_{a}^{(i)}\subseteq{\Cc}_{a}^{(i)}$ thus $N_{F_k}(b,\bar{\Cc}_{a}^{(i)})\leq N_{F_k}(b,{\Cc}_{a}^{(i)})$, and (\ref{eqn:cont_D23}) follows from Lemma~\ref{lemma:cont_freq}.
\end{proof}

\begin{Lemma}\label{lemma:cont_sigma}
For any arm $a\in[K]$, and any time slot $t\in\Dc_T$, we have
$\hat{\sigma}_t(a)^2\leq \frac{4\delta^2}{\frac{tp}{16}-\frac{8K\alpha_t^2}{\Delta^2}d\log  \left(\frac{d+t}{d}\right)}$, 
where $\delta=l\sqrt{d/\lambda_0(\{\Phi_a\})}$.
\end{Lemma}
\begin{proof}
Let 
\begin{align}\label{eqn:hat_z}
\hat{\zv}_t^{(i)}({a})&:=\frac{\sum_{\tau \in \Tc_t(a,\bar{\Cc}_{a}^{(i)})} \xv(a, c_\tau)}{N_t(a,\bar{\Cc}_{a}^{(i)})}
\end{align}
be the empirical average of the feature vectors over the time slots before $t$ when arm $a$ is pulled under a context in $\bar{\Cc}_{a}^{(i)}$. Since $B(\zv_{a}^{(i)},r)$ is convex, $\hat{\zv}_t^{(i)}(a)\in B(\zv_{a}^{(i)},r)$. Thus, according to Lemma~\ref{lemma:span3}, $\{\hat{\zv}_t^{(i)}(a)\}_i$ form a valid basis for $\Xc_a$. Let $\hat{\Zv}_t(a)$ be the matrix whose columns are $\hat{\zv}_{t}^{(i)}(a)$. Then, we can always obtain a vector $\bar{\betav}$, such that
$\xv(a,c_t)=\hat{\Zv}_t(a)\bar{\betav}$. Besides, 
\begin{align}
\|\bar{\betav}\|_1\leq \frac{l\sqrt{d}}{\sqrt{\lambda_0(\{\Phi_a\})}-\sqrt{d}r}=2\delta.\label{eqn:cont_bar_beta}
\end{align}

Expanding $\hat{\zv}_{t}^{(i)}(a)$, we have
\begin{align}
\xv(a,c_t)=\sum_{i=1}^d \frac{\sum_{\tau \in \Tc_t(a,\bar{\Cc}_{a}^{(i)})} \xv(a, c_\tau)}{N_t(a,\bar{\Cc}_{a}^{(i)})} \bar{\betav}[i],
\end{align} 
i.e., $\xv(a,c_t)$ can be expressed as a linear combination of the feature vectors $\{\xv(a, c_\tau)\}$ for $\tau \in \cup_i\Tc_t(a,\bar{\Cc}_{a}^{(i)})$, where the corresponding coefficients are $\bar{\betav}[i]/N_t(a,\bar{\Cc}_{a}^{(i)})$.

Thus, according to Proposition~\ref{prop:beta}, we have
\begin{align}
\hat{\sigma}_t(a)^2 & \leq \sum_{i=1}^d \frac{\bar{\betav}[i]^2}{N(a, \bar{\Cc}_{a}^{(i)})}\leq \frac{\sum_{i=1}^d\bar{\betav}[i]^2}{\frac{tp}{16}-\frac{8K\alpha_t^2}{\Delta^2}d\log  \left(\frac{d+t}{d}\right)}\leq \frac{\|\bar{\betav}\|_1^2}{\frac{tp}{16}-\frac{8K\alpha_t^2}{\Delta^2}d\log  \left(\frac{d+t}{d}\right)}\label{eqn:cont_D3}\\
&\leq \frac{4\delta^2}{\frac{tp}{16}-\frac{8K\alpha_t^2}{\Delta^2}d\log  \left(\frac{d+t}{d}\right)},\label{eqn:cont_D4}
\end{align}
where (\ref{eqn:cont_D3}) follows from Lemma~\ref{lemma:cont_fraction}, and (\ref{eqn:cont_D4}) follows from (\ref{eqn:cont_bar_beta}).
\end{proof}

\begin{theorem}\label{thm:Dt}
Let  
\begin{align*}
t_3 & = \max \left\{ \frac{1728(2+d)(\delta^2+2Kd)}{\Delta^2p}, 10 \right\} ,\quad t_4  = \max\left\{t_3 \log^2  t_3 ,  \exp \left(\frac{12l^2s^2}{2+d}\right)\right\}.
\end{align*}
Then, under Algorithm~\ref{alg:linUCB-d}, 
\begin{align*}
\Eb[R(\Dc_T)]&\leq t_4 M = O\left(\frac{d(\delta^2+2Kd)}{\Delta^2p}\log^2 \left(\frac{d(\delta^2+2Kd)}{\Delta^2p}\right)\right). 
\end{align*}
\end{theorem}
\begin{proof}
First, we note that
\begin{align}
\frac{t_4}{(\log t_4)^2}&\tcr{\geq}\frac{t_3\log^2t_3}{(\log t_3+2\log\log t_3)^2}>\frac{t_3\log ^2t_3}{9\log^2t_3}=\frac{t_3}{9}.
\end{align}
Thus, for any $t\geq t_4$, we have
\begin{align}
\frac{t}{(\log t)^2}&\geq\frac{t_3}{9}= \frac{192(2+d)(\delta^2+2Kd)}{\Delta^2p},
\end{align}
which is equivalent to
\begin{align}\label{eqn:t4}
t&\geq \frac{192(2+d)(\delta^2+2Kd)}{\Delta^2p}\log^2 t.
\end{align}
According to (\ref{eqn:alpha_t}), $3(2+d)\log t\geq \alpha_t^2$ for $t>t_4$. Thus, (\ref{eqn:t4}) can be further bounded as
\begin{align}
t&\geq \frac{64\alpha_t^2}{\Delta^2p}(\delta^2+2Kd) \log t\\
&\geq  \frac{64\alpha_t^2}{\Delta^2p}\left(\delta^2+2Kd \log \frac{d+t}{d} \right),
\end{align}
where the last inequality follows from the fact that when $d>1$, $\log t\geq \log \frac{d+t}{d}$. 
Rearranging the terms, we have
\begin{align}
\Delta^2&> \frac{4\alpha_t^2\delta^2 }{\frac{tp}{16}-\frac{8K\alpha_t^2}{\Delta^2}d\log  \left(\frac{d+t}{d}\right)}\geq 4\alpha_t^2\hat{\sigma}_t(a)^2,\quad \forall a\in[K],
\end{align}
where the last inequality follows from Lemma~\ref{lemma:cont_sigma}.

Thus, $\Dc_T$ can only include time slots $t< t_4$.
The bound on $\Eb[R(\Dc_T)]$ then follows.
\end{proof}

\subsection{Put Everything Together}
After obtaining bounds on the expected regret over $\Ac_T$, $\Bc_T$, $\Cc_T$ and $\Dc_T$, we are ready to obtain the result in Theorem~\ref{thm:continuous}. We have
\begin{align}
\Eb[R_T]&\leq \Eb[R(\Ac_T)]+\Eb[R(\Bc_T)]+\Eb[R(\Cc_T)]+\Eb[R(\Dc_T)]\nonumber\\
&\leq \frac{8MKd}{p^2} +\big(\frac{8}{p}+1\big)  M (2+2.5K)+t_4 M \\
&=O\left(\frac{Kd}{p^2}+\frac{d(2\delta^2+Kd)}{\Delta^2p}\log^2 \left(\frac{d(2\delta^2+Kd)}{\Delta^2p}\right)\right) .
\end{align}
\end{document}